\documentclass{article}

\usepackage{amsmath,amsfonts,bm}

\def\eqref#1{equation~\ref{#1}}

\def\1{\bm{1}}

\def\rvb{{\mathbf{b}}}

\def\rvw{{\mathbf{w}}}
\def\rvx{{\mathbf{x}}}
\def\rvy{{\mathbf{y}}}
\def\rvz{{\mathbf{z}}}

\def\vtheta{{\bm{\theta}}}

\def\vb{{\bm{b}}}

\def\vs{{\bm{s}}}

\def\vv{{\bm{v}}}

\def\vy{{\bm{y}}}

\def\mA{{\bm{A}}}
\def\mB{{\bm{B}}}
\def\mC{{\bm{C}}}

\def\mF{{\bm{F}}}
\def\mG{{\bm{G}}}
\def\mH{{\bm{H}}}
\def\mI{{\bm{I}}}

\def\mP{{\bm{P}}}

\def\mS{{\bm{S}}}

\def\mPhi{{\bm{\Phi}}}

\def\mSigma{{\bm{\Sigma}}}

\DeclareMathAlphabet{\mathsfit}{\encodingdefault}{\sfdefault}{m}{sl}
\SetMathAlphabet{\mathsfit}{bold}{\encodingdefault}{\sfdefault}{bx}{n}

\def\gN{{\mathcal{N}}}

\newcommand{\R}{\mathbb{R}}

\usepackage[final]{neurips_2024}

\usepackage[utf8]{inputenc} %
\usepackage[T1]{fontenc}    %
\usepackage{hyperref}       %
\usepackage{url}            %
\usepackage{booktabs}       %
\usepackage{amsfonts}       %
\usepackage{nicefrac}       %
\usepackage{microtype}      %
\usepackage{xcolor}         %
\usepackage{amsthm}
\usepackage{multirow}
\usepackage{array}
\usepackage{graphicx}
\usepackage{caption}
\usepackage{subcaption}
\usepackage{algorithm}
\usepackage{algpseudocode}
\usepackage[export]{adjustbox}
\usepackage{wrapfig}
\usepackage{tcolorbox}
\usepackage{algorithm}
\usepackage{algpseudocode}

\title{Fast Samplers for Inverse Problems \\ in Iterative Refinement Models}

\author{%
  Kushagra Pandey$^*$ \\
  Department of Computer Science\\
  University of California Irvine\\
  \texttt{pandeyk1@uci.edu}
  \And
  Ruihan Yang\thanks{Equal contribution} \\
  Department of Computer Science \\
  University of California Irvine \\
  \texttt{ruihan.yang@uci.edu}
  \AND
  Stephan Mandt \\
  Department of Computer Science \\
  University of California Irvine \\
  \texttt{mandt@uci.edu} 
}

\theoremstyle{definition}

\newtheorem{proposition}{Proposition}
\newtheorem*{proposition*}{Proposition}

\newtheorem*{theorem*}{Theorem}

\definecolor{algoshade}{HTML}{dedbd2}
\definecolor{algoshade2}{HTML}{e3d5ca}
\definecolor{algoshade3}{HTML}{dee2e6}

\tcbset{
  mathstyle/.style={
    colback=algoshade3,  %
    colframe=white,
    boxrule=0.5mm,
    arc=0mm,
    auto outer arc,
    boxsep=1pt,
    left=2pt,
    right=2pt,
    top=1pt,
    bottom=1pt
  }
}
\tcbset{
  algostyle/.style={
    colback=algoshade3,  %
    colframe=white,
    boxrule=0.5mm,
    arc=0mm,
    auto outer arc,
    boxsep=0pt,
    left=0pt,
    right=1pt,
    top=1pt,
    bottom=1pt
  }
}
\tcbset{
  algostyle2/.style={
    colback=algoshade2,  %
    colframe=white,
    boxrule=0.5mm,
    arc=0mm,
    auto outer arc,
    boxsep=0pt,
    left=0pt,
    right=1pt,
    top=1pt,
    bottom=1pt
  }
}
\tcbset{
  algostyle3/.style={
    colback=algoshade,  %
    colframe=white,
    boxrule=0.5mm,
    arc=0mm,
    auto outer arc,
    boxsep=0pt,
    left=0pt,
    right=1pt,
    top=1pt,
    bottom=1pt
  }
}

\begin{document}

\maketitle

\begin{abstract}
    Constructing fast samplers for unconditional diffusion and flow-matching models has received much attention recently; however, existing methods for solving \emph{inverse problems}, such as super-resolution, inpainting, or deblurring, still require hundreds to thousands of iterative steps to obtain high-quality results. We propose a plug-and-play framework for constructing efficient samplers for inverse problems, requiring only \emph{pre-trained} diffusion or flow-matching models. We present \emph{Conditional Conjugate Integrators}, which leverage the specific form of the inverse problem to project the respective conditional diffusion/flow dynamics into a more amenable space for sampling. Our method complements popular posterior approximation methods for solving inverse problems using diffusion/flow models. We evaluate the proposed method's performance on various linear image restoration tasks across multiple datasets, employing diffusion and flow-matching models. Notably, on challenging inverse problems like 4$\times$ super-resolution on the ImageNet dataset, our method can generate high-quality samples in as few as \emph{5} conditional sampling steps and outperforms competing baselines requiring 20-1000 steps. Our code will be publicly available at \url{https://github.com/mandt-lab/c-pigdm}.
\end{abstract}
 \section{Introduction}
\label{sec:intro}
Iterative refinement models, such as diffusion generative models and flow matching methods \citep{sohl2015deep, ho2020denoising, songscore, lipman2023flow, albergo2023building}, have seen increasing popularity in recent months, and much effort has been invested in accelerating unconditional sampling in these models \citep{pandey2024efficient, shaul2023bespoke, sauer2024fast, karraselucidating, salimansprogressive, zhang2023fast, lu2022dpm, songdenoising}. However, while most efficient samplers have been designed in the \emph{unconditional} setup, current methods for solving \emph{inverse} problems, such as deblurring, inpainting, or super-resolution, still require hundreds to thousands of neural network evaluations to achieve the highest perceptual quality. Moreover, in addition to a score function evaluation, a class of existing methods for solving inverse problems using pre-trained unconditional iterative refinement models often involves expensive Jacobian-vector products \citep{song2022pseudoinverse, chung2022diffusion}, making a single sampling step quite expensive and therefore, intolerably slow for most practical applications. 

This paper presents a principled framework for designing efficient samplers for guided sampling in iterative refinement models, accelerating existing samplers like $\Pi$GDM by an order of magnitude. We present our framework for inverse problems where the degradation operator is known and might be corrupted with additional noise. Crucially, our transformations do not require any re-training and merely rely on some algebraic manipulations of the equations to be simulated. 

\begin{figure}
    \centering
    \includegraphics[width=1.0\linewidth]{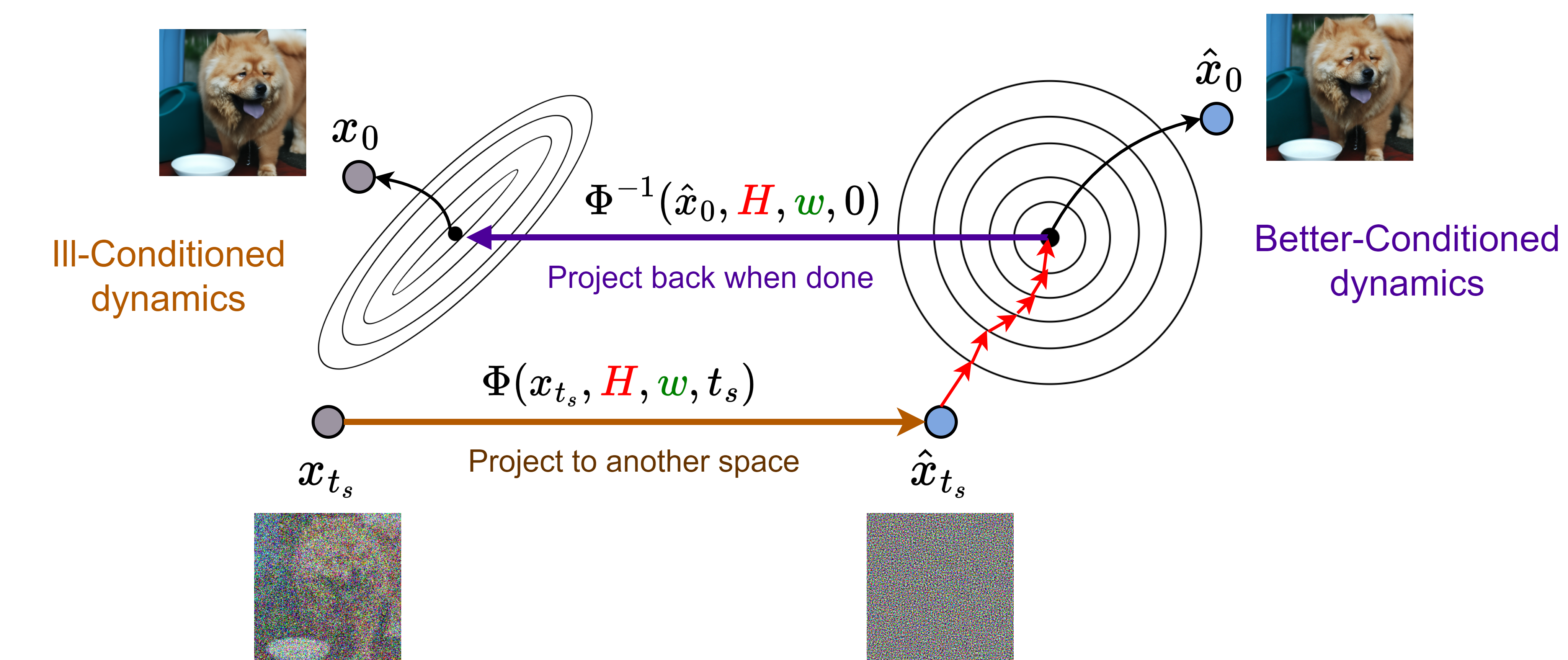}
    \caption{Illustration of Conditional Conjugate Integrators for Fast Sampling in Inverse Problems. Given an initial sampling latent $\rvx_{t_s}$ at time $t_s$, our sampler projects the diffusion/flow dynamics to a more amenable space for sampling using a projector operator $\Phi$ which is conditioned on the degradation operator $\mH$ and the sampling guidance scale $w$. The diffusion/flow sampling is then performed in the projected space. Post completion, the generated sample in the projected space is transformed back into the original space using the inverse of the projection operator, yielding the final generated sample. We define the form of the operator $\Phi$ in Section \ref{sec:cci_diffusion}. Conditional Conjugate Integrators can significantly speed up sampling in challenging inverse problems and can generate high-quality samples in as few as 5 NFEs as compared to existing baselines, which require from 20-1000 NFEs (see Section \ref{sec:experiment}).}
    \label{fig:main}
\end{figure}

Intuitively, we expand on the concept of Conjugate Integrators \citep{pandey2024efficient} by projecting the conditional generation process in inverse problems to another space that might be better conditioned for faster sampling (See Figure \ref{fig:main}). To this end, we separate the linear and non-linear components in the generation process and parameterize the transformation by analytically solving the linear coefficients. By the end of the sampling procedure, we map back to the original sampling space, leading to the concept of \emph{Conditional Conjugate Integrators} that apply to various iterative refinement models such as diffusion models, flows, and interpolants. 

In more detail, our main contributions are as follows.

\begin{itemize}
    \item \textbf{Conditional Conjugate Integrators:} We repurpose the recently proposed Conjugate Integrator framework \citep{pandey2024efficient} for fast guided sampling in iterative refinement models (diffusions and flows) for linear inverse problems and refer to it as \emph{Conditional Conjugate Integrators}. Next, we design a specific parameterization of the proposed framework, which encodes the structure of the linear inverse problem in the sampler design itself.
    \item \textbf{Theoretical Analysis:} Our parameterization exhibits theoretical properties that help us identify key parameters for sampler design. More specifically, we show that our parameterization (by design) enables recovering high-frequency details early on during sampling. This further enables fast-guided sampling while maintaining good sample quality in the context of inverse problems.
    \item \textbf{Empirical Results}. Empirically, we show that our proposed sampler significantly improves over baselines in terms of sampling efficiency on challenging benchmarks across inverse problems like super-resolution, inpainting, and Gaussian deblurring. For instance, on a challenging 4x superresolution task on the ImageNet dataset, \emph{our proposed sampler achieves better sample quality at 5 steps, compared to 20-1000 steps required by competing baselines}. Additionally, we extend the proposed framework for noisy and non-linear inverse problems with qualitative demonstrations.
\end{itemize}

\section{Fast Samplers for Inverse Problems using Diffusions/Flows.}
\label{sec:method}
\subsection{Background and Problem Statement}
Diffusion models define a continuous-time \textit{forward process} (usually with an affine drift)
to convert data $\rvx_0 \in \R^d$ into noise. A learnable \textit{reverse} process is trained to generate data from noise. In this work, we only consider deterministic reverse processes specified as an ODE~\citep{songscore},
\begin{equation}
    d \rvx_t = \left[\mF_t\rvx_t - \frac{1}{2}\mG_t \mG_t^\top \nabla_{\rvx_t} \log p_t(\rvx_t)\right] \, dt.
    \label{eqn:prob_flow}
\end{equation}
The score is usually intractable and is approximated using a parametric estimator $\vs_{\theta}(\rvx_t, t)$, trained using denoising score matching \citep{6795935, song2019generative, songscore}. 
Analogously, one-sided stochastic interpolants \citep{albergo2023building} define an \emph{interpolant}\footnote{In this work, we use the terms interpolant and flows interchangeably.} $\rvx_t = \alpha_t \rvx_1 + \gamma_t \rvz$, where $\quad \rvx_1 \sim p_\text{data}$ ,and $\quad \rvz \sim \mathcal{N}(0, \mI)$
to define a transport map between the generative prior (typically an isotropic Gaussian) and the data distribution. Interestingly, the one-sided interpolant induces a vector field $\vb(\rvx_t, t) = \mathbb{E}[\dot{\alpha}_t\rvx_1 + \dot{\gamma}_t \rvz | \rvx_t]$,
where $\dot{\alpha}_t$, $\dot{\gamma}_t$ represent the time derivatives of $\alpha_t$ and $\gamma_t$, respectively. The vector field $\vb(.)$ is typically learned using a neural network approximation $\vb_\theta(\rvx_t, t)$. The deterministic interpolant process can then be specified as $d \rvx_t = \vb_\theta(\rvx_t, t)\, dt$.
Numerically solving these deterministic generative processes with a sufficient sampling budget can generate plausible samples from noise.

\textbf{Problem Statement.} Given a \textit{noisy linear degradation process} (we will consider non-linear processes later) with a degradation operator $\mH$ specified over an \emph{unobserved} data point $\rvx_0$,
\begin{equation}
    \rvy = \mH\rvx_0 + \sigma_y \rvz, \quad \rvz \sim \mathcal{N}(0, \mI), \,\, \rvx_0 \sim p_\text{data},
    \label{eqn:y_likelihood}
\end{equation}
the goal is to recover the original signal $\rvx_0$. Additionally, given an unconditional pre-trained diffusion or flow matching model, one approach for solving inverse problems is to infer the posterior distribution over the data given the degraded observation, i.e., $p(\rvx_0|\rvy) \propto p(\rvy|\rvx_0)p(\rvx_0)$ by simulating the conditional reverse process dynamics i.e.
\begin{align}
    \text{Diffusion:}\qquad\qquad& d \rvx_t = \Big[\mF_t\rvx_t - \frac{1}{2}\mG_t \mG_t^\top \nabla_{\rvx_t} \log p(\rvx_t|\vy)\Big] dt, \label{eqn:cond_diff}\\
    \text{Flows:} \qquad\qquad& d \rvx_t = \vb(\rvx_t, \vy, t) dt, \nonumber
\end{align}
where $\nabla_{\rvx_t} \log p(\rvx_t|\vy)$ and $\rvb(\rvx_t,\rvy,t)$ are the conditional score and velocity estimates, respectively.
One approach could be to directly model the conditional score or velocity estimates using a conditional iterative refinement model \citep{saharia2022image, saharia2022palette}. However, such approaches are problem-dependent, requiring expensive training pipelines to account for the lack of generalization across inverse problems. Additionally, such methods rely on the availability of paired $(\rvx_t, \vy)$ measurements, which can be expensive to acquire. Alternatively, \textit{problem-agnostic} methods leverage pre-trained unconditional iterative refinement models to estimate the conditional score or velocity fields and can generalize to different inverse problems without extra training. In this work, we restrict our discussion to the latter and discuss estimating conditional score/velocity fields next.

\textbf{Estimating Conditional Score/Velocity from Pretrained Models:} For diffusion models, approximating the conditional score follows directly from Bayes Rule, i.e. $\nabla_{\rvx_t} \log p(\rvx_t|\vy) \approx \vs_\theta(\rvx_t, t) + w_t\nabla_{\rvx_t} \log p(\rvy|\rvx_t)$ where $w_t$ is the \textit{guidance weight} (or temperature) of the distribution $p(\rvy|\rvx_t)$. Analogously for interpolants (or flows), \citet{pokle2024trainingfree} propose the conditional flow dynamics,
 \begin{equation}
     \vb(\rvx_t, \rvy, t) \approx \vb_\theta(\rvx_t, t) + w_t\frac{\gamma_t}{\alpha_t}\Big[\gamma_t\dot{\alpha}_t - \dot{\gamma}_t\alpha_t\Big]\nabla_{\rvx_t} \log p(\rvy|\rvx_t). \label{eqn:flow_cond}
 \end{equation}
We include a formal proof for the result in Eqn. \ref{eqn:flow_cond} from an interpolant perspective in Appendix \ref{app:proof_1}. Since the conditional score and velocity estimates require approximating the term $\nabla_{\rvx_t} \log p(\rvy|\rvx_t)$, we discuss its estimation next.

\textbf{Estimation of the Noise Conditional Score $\nabla_{\rvx_t} \log p(\rvy|\rvx_t)$:} The noise conditional distribution $p(\vy|\rvx_t)$ can be represented as $p(\rvy|\rvx_t) = \int p(\rvy|\rvx_0)p(\rvx_0|\rvx_t)d\rvx_0$. For problem-agnostic models, it is common to approximate the posterior $p(\rvx_0|\rvx_t)$ using an unimodal Gaussian distribution \citep{chung2022diffusion, song2022pseudoinverse}. In this work, we restrict our discussion to the posterior approximation in $\Pi$GDM \citep{song2022pseudoinverse} and its flow variant~\citep{pokle2024trainingfree} (named as $\Pi$GFM in our work), $p(\rvx_0|\rvx_t) \approx \mathcal{N}(\hat{\rvx}_0, r_t^2\mI_d)$, which yields the following estimate of the conditional score:
\begin{equation}
    \nabla_{\rvx_t} \log p(\rvy|\rvx_t) = \frac{\partial \hat{\rvx}_0}{\partial \rvx_t}^\top \mH^\top(r_t^2\mH\mH^\top + \sigma_y^2 \mI_d)^{-1} (\vy - \mH\hat{\rvx}_0), \label{eqn:pigdm_approx}
\end{equation}
where $\hat{\rvx}_0$ is the first-order Tweedie's moment estimate \citep{stein1981estimation}. Our choice of using the $\Pi$GDM approximation is motivated by its expressive posterior approximation $p(x_0|x_t)$ compared to other methods such as DPS or MCG. This makes it an excellent starting point for low-budget sampling.

\subsection{Conditional Conjugate Integrators}
\label{sec:cci_diffusion}
\paragraph{Conjugate Integrators}
The main idea in conjugate integrators \citep{pandey2024efficient} is to project the diffusion dynamics in Eqn. \ref{eqn:prob_flow} into another space where sampling might be more efficient. The projected diffusion dynamics can then be solved using any numerical ODE solver. On completion, the dynamics can be projected back to the original space to generate samples from the data distribution. To this end, \citet{pandey2024efficient} introduce an invertible time-dependent affine transformation $\bar{\rvx}_t = \mA_t\rvx_t$. Interestingly, conjugate samplers have theoretical connections to prior work in fast sampling for unconditional diffusion models \citep{songdenoising, zhang2023fast, lu2022dpm}. We refer the readers to \citet{pandey2024efficient} for exact details.

\subsubsection{Conjugate Integrators for Inverse Problems}
Next, we design conjugate integrators for linear inverse problems. For simplicity, we discuss noiseless inverse problems, $\sigma_y=0$, and defer the discussion of noisy inverse problems to Section \ref{sec:noisy_nl}. Furthermore, due to space constraints, we present our analysis for diffusion models and defer the discussion of flows to Appendix \ref{app:cci_flows}. Lastly, without loss of generality, we assume the standard score network parameterization, $\vs_\theta(\rvx_t, t) = \mC_\text{out}(t)\bm{\epsilon}_\theta(\rvx_t, t)$ where $\mC_\text{out}(t)$ is the notation from the score precondition defined in ~\citet{karraselucidating}.

 A straightforward way to define conditional conjugate integrators is to treat the score estimate $\nabla_{\rvx_t} \log p(\rvy|\rvx_t)$ as a \emph{black-box} i.e., ignore the structure of the inverse problem. For this case, we formally specify the conjugate integrator formulation as,

\begin{proposition}
     (Extended $\Pi$GDM) \textit{For the conditional diffusion dynamics defined in Eqn. \ref{eqn:cond_diff}, introducing a diffeomorphism, $\bar{\rvx}_t=\mA_t\rvx_t$, where,
    \begin{equation}
     \mA_t = \bm{\exp}{\left(\int_0^t \mB_s - \mF_s ds\right)}, \quad\quad \bm{\Phi}_t = -\int_0^t \frac{1}{2}\mA_s\mG_s\mG_s^\top\mC_\text{out}(s) ds,
    \end{equation}
    induces the following projected diffusion dynamics,
    \begin{equation}
    d\hat{\rvx}_t= \mA_t\mB_t\mA_t^{-1}\hat{\rvx}_t dt + d\bm{\Phi}_t\bm{\epsilon}_{\vtheta}\left(\rvx_t, t\right) - \frac{w_tr_t^{-2}}{2}\mG_t \mG_t^\top \frac{\partial \hat{\rvx}_0}{\partial \rvx_t}^\top (\mH^\dag\vy - \mP\hat{\rvx}_0) dt,
    \label{eqn:c_pigdm}
    \end{equation}
 where $\mH^\dag = \mH^\top(\mH\mH^\top)^{-1}$ and $\mP=\mH^\top(\mH\mH^\top)^{-1}\mH$ represent the pseudoinverse and the orthogonal projector operators for the degradation operator $\mH$.  (Proof in Appendix \ref{app:proof_2})}
 \label{prop:2}
\end{proposition}
Similar to \citet{pandey2024efficient}, the matrix $\mB_t$ is a design choice. We refer to the formulation in Eqn. \ref{eqn:c_pigdm} as Extended $\Pi$GDM since for $\mB_t=0$, the ODE in Eqn. \ref{eqn:c_pigdm} becomes equivalent to the $\Pi$GDM formulation proposed in \citet{song2022pseudoinverse}. This is because, for $\mB_t=0$, Conjugate Integrators are equivalent to DDIM \citep{songdenoising} (See \citet{pandey2024efficient} for proof). Therefore, the projected diffusion dynamics in Eqn. \ref{eqn:c_pigdm} already present a more generic framework for designing samplers for inverse problems over $\Pi$GDM.
In this work, we only explore the parameterization in Eqn. \ref{eqn:c_pigdm} for $\mB_t=0$ and hence refer to it simply as \emph{$\Pi$GDM} (analogously \emph{$\Pi$GFM} for flows; see Appendix \ref{app:cci_flows}). 

One characteristic of the formulation in Eqn. \ref{eqn:c_pigdm} is the black-box nature of the conditional score $\nabla_{\rvx_t} \log p(\rvy|\rvx_t)$. However, the inherent linearity in the conditional score can be used to design \emph{better conditioned} (more on this in Section \ref{sec:theoretical_aspects}) conjugate integrators, which we illustrate formally in the form of the following result.

\begin{proposition}
    (Conjugate $\Pi$GDM) \textit{Given a noiseless linear inverse problem with $\sigma_y=0$, a design matrix $\mB: [0,1] \rightarrow \mathbb{R}^{d\times d}$, and the conditional score $\nabla_{\rvx_t} \log p(\rvy|\rvx_t)$ approximated using Eqn. \ref{eqn:pigdm_approx}, introducing the transformation $\bar{\rvx}_t=\mA_t\rvx_t$, where
    \begin{equation}
    \mA_t = \bm{\exp} \Big[\int_0^t \mB_s - \Big(\mF_s + \frac{w_s r_s^{-2}}{2\mu_s^2}\mG_s \mG_s^\top\mP\Big)ds\Big],
    \label{eqn:conj_at}
    \end{equation}
    induces the following projected diffusion dynamics:
    \begin{tcolorbox}[mathstyle]
    \begin{equation}
        d\bar{\rvx}_t = \mA_t\mB_t\mA_t^{-1}\bar{\rvx}_t dt + d\mPhi_y \vy + d\mPhi_s \bm{\epsilon}_\theta(\rvx_t, t) + d\mPhi_j \Big[\partial_{\rvx_t} \bm{\epsilon_\theta(\rvx_t, t)} (\mH^\dag\vy - \mP\hat{\rvx}_0)\Big],
        \label{eqn:imp_proj}
    \end{equation}
    \end{tcolorbox}
    where $\exp(.)$ denotes the matrix exponential, $\mH^\dag$, and $\mP$ are the pseudoinverse and projector operators (as defined previously). Proof in Appendix \ref{app:proof_3}.}
    \label{theorem:1}
\end{proposition}

In this case, the coefficients $\mPhi_y$, $\mPhi_j$, and $\mPhi_s$ depend on time $t$ and the degradation operator $\mH$ (See Appendix \ref{app:proof_3} for full definitions). Intuitively, by including information about the degradation operator $\mH$ and the guidance scale in the transformation $\mA_t$ in Eqn. \ref{eqn:conj_at}, we incorporate the specific structure of the inverse problem in the sampler design, which can have several advantages (more on this in Section \ref{sec:theoretical_aspects}). Moreover, the matrix $\mB_t$ is a design choice of our parameterization (we will discuss exact choices in Section \ref{sec:in_the_wild}). We refer to this parameterization as \emph{C-$\Pi$GDM} (analogously \emph{C-$\Pi$GFM for flows; see Appendix \ref{app:cci_flows}}). In this work, we restrict our discussion to this parameterization and discuss some practical and theoretical aspects next.

\subsubsection{Practical Design Choices}
\label{sec:in_the_wild}
\textbf{Choice of Diffusions and Flows:} While our proposed integrators are applicable to generic diffusion processes \citep{dockhornscore, pandey2023generative} and flows \citep{ma2024sit}, we restrict follow-up discussion to VP-SDE \citep{songscore} diffusion for which $\mF_t=-\frac{1}{2}\beta_t \mI_d, \mG_t = \sqrt{\beta_t}\mI_d$ and OT-flows \citep{liu2022flow, lipman2023flow} for which $\alpha_t=t, \gamma_t=1-t$. For our score network parameterization, we set $\mC_\text{out}(t) = -1/\sigma_t$, corresponding to the standard $\epsilon$-prediction \citep{ho2020denoising, songscore} parameterization in diffusion models.

\textbf{Choice of $\mB_t$:} Similar to \citet{pandey2024efficient}, we set $\mB_t=\lambda \mI_d$, where $\lambda$ is a time-invariant scalar hyperparameter tuned during inference for optimal sample quality.

\textbf{Choice of $w_t$:} Similar to prior work \citep{song2022pseudoinverse, pokle2024trainingfree}, we use an adaptive guidance weight schedule. For diffusion models, we use $w_t=w \mu_t^2 r_t^2$
where $r_t^2=\frac{\sigma_t^2}{\mu_t^2 + \sigma_t^2}$. Analogously, for flows, we set $w_t=w \alpha_t^2 r_t^2$ where $r_t^2=\frac{\gamma_t^2}{\alpha_t^2 + \gamma_t^2}$

Having an extra multiplicative factor of $\mu_t^2$ (for VP-SDE) or $\alpha_t^2$ (for flows) stabilizes the numerical computation of coefficients in Eqn. \ref{eqn:imp_proj} before sampling. We tune the static guidance weight $w$ during inference for optimal sample quality.

\textbf{Choice of Start Time:} Given a degradation output $\vy$, it is common to start diffusion or flow sampling at $\tau<T$ or $\tau>0$, respectively \citep{chung2022comecloserdiffusefaster, song2022pseudoinverse, pokle2024trainingfree}. Consequently, we initialize the diffusion sampling process as $\rvx_\tau = \mu_\tau \mH^\dag \vy + \sigma_\tau \rvz$. Analogously for flows, we initialize sampling at $\rvx_\tau = \alpha_\tau \mH^\dag \vy + \gamma_\tau \rvz$.

\textbf{Choice of the ODE Solver:} Unless specified otherwise, we use the Euler discretization scheme for C-$\Pi$G(D/F)M samplers.

We illustrate a generic C-$\Pi$GDM sampling routine in Algorithm \ref{algo:cpigdm_algo} and include additional implementation details in Appendix \ref{app:in_the_wild}. Next, we present some theoretical aspects of our proposed method.

\begin{algorithm}[t]
\small
\caption{{\textit{Conjugate $\Pi$GDM sampling}}}
\begin{algorithmic}[1]

\State {\bfseries Input:} Corrupted observation $y$, Corruption operator $\mH$, Denoiser $\bm{\epsilon}_\vtheta(.,.)$, Choice of $\mB_t$, NFE budget $N$, Timestep discretization $\{t_i\}_{i=0}^N$, Diffusion kernel $p(\rvx_t|\rvx_0) = \gN(\mu_t\rvx_0, \sigma_t^2\mI_d)$, Start time $\tau$.
\State {\bfseries Output:} Clean sample $\hat{\rvx}_{0}$
\begin{tcolorbox}[algostyle3]
    \State Pre-Compute $\{\mA_{t_i}\}_{i=0}^N$ (Eqn. \ref{eqn:conj_at}) \Comment{Pre-compute coefficients}
\State Pre-Compute $\{\mPhi_y^i, \mPhi_s^i, \mPhi_j^i\}_{i=0}^N$ (see App. \ref{app:proof_3})
\end{tcolorbox}
\begin{tcolorbox}[algostyle2]
\State $\rvz \sim p(\rvx_T)$ \Comment{Draw initial samples from the generative prior}
\State $\rvx = \mu_\tau \mH^\dag \vy + \sigma_\tau \rvz$ \Comment{Initialize using the pseudoinverse (See \citet{chung2022comecloserdiffusefaster})}
\State $\bar{\rvx} = \mA_\tau \rvx$ \Comment{Initial Projection Step}
\end{tcolorbox}
\begin{tcolorbox}[algostyle]
    \For{$n=0$ {\bfseries to} $N-1$}
\State $h = (t_{n+1} - t_n)$ \Comment{Time step differential}
\State $\rvx = \mA_{t_n}^{-1}\bar{\rvx}$
\State $\hat{\rvx}_0 = \frac{1}{\mu_{t_n}}[\rvx - \sigma_{t_n}\epsilon_\theta(\rvx, t_n)]$ \Comment{Tweedie's Estimate}
\State $\vv_l = h\mA_{t_n}\mB_{t_n}\mA_{t_n}^{-1}\bar{\rvx} + (\mPhi_y^{n+1} - \mPhi_y^{n})\vy$ \Comment{Linear drift}
\State $\vv_{nl} = (\mPhi_s^{n+1} - \mPhi_s^{n})\boldsymbol{\epsilon}_\vtheta(\rvx, t_n) + (\mPhi_j^{n+1} - \mPhi_j^{n})\Big[\partial_{\rvx} \boldsymbol{\epsilon}_\vtheta(\rvx, t_n) (\mH^\dag\vy - \mP\hat{\rvx}_0)\Big] $ \Comment{Non-Linear drift}
\State $\bar{\rvx} = \bar{\rvx} + \vv_l + \vv_{nl}$ \Comment{Euler Update}
\EndFor
\end{tcolorbox}
\noindent\Return $\rvx = \mA_{t_N}^{-1}\bar{\rvx}$ \Comment{Project back to original space when done}
\end{algorithmic}
\label{algo:cpigdm_algo}
\end{algorithm}

\subsection{Theoretical Aspects}
\label{sec:theoretical_aspects}
With the simplifications in Section \ref{sec:in_the_wild}, the transformation $\mA_t$ in Eqn. \ref{eqn:conj_at} simplifies to:
\begin{equation}
    \mA_t = \bm{\exp} \Big[\int_0^t \Big(\lambda + \frac{1}{2}\beta_s\Big) ds \mI_d -\frac{w}{2}\Big(\int_0^t \beta_s ds\Big)\mP\Big]
    \label{eqn:vp_conj_at},
\end{equation}
where $\mP=\mH^\top(\mH\mH^\top)^{-1}\mH$ is an orthogonal projection operator.

\textbf{Computing $\mA_t$:} While computing the matrix exponential in Eqn. \ref{eqn:vp_conj_at} might seem non-trivial, it has several interesting properties that make it tractable to compute. More specifically, the matrix exponential in Eqn. \ref{eqn:vp_conj_at} can be simplified as (Proof in Appendix \ref{app:proof_4}),
\begin{equation}
    \mA_t = \exp(\kappa_1(t)) \Big[\mI_d + (\exp(\kappa_2(t)) - 1)\mP\Big], \quad \kappa_1(t) = \int_0^t \Big(\lambda + \frac{1}{2}\beta_s\Big) ds,\quad\kappa_2(t)=-\frac{w}{2}\int_0^t \beta_s ds,
    \label{eqn:at_simplified}
\end{equation}
where $\exp(.)$ in Eqn. \ref{eqn:at_simplified} represents the scalar exponential. Furthermore, the integrals in Eqn. \ref{eqn:at_simplified} are trivial to compute analytically or numerically, making $\mA_t$ easier to compute. Moreover, $\mA_t^{-1}$ can also be compactly represented as,
\begin{equation}
    \mA_t^{-1} = \exp(-\kappa_1(t)) \Big[\mI_d + (\exp(-\kappa_2(t)) - 1)\mP\Big],
\end{equation}
and is also tractable to compute. Due to the tractability of $\mA_t$ and $\mA_t^{-1}$, the projected diffusion dynamics in C-$\Pi$GDM are straightforward to simulate numerically.

\textbf{Intuition behind $\mA_t$:} Next, we analyze several theoretical properties of the transformation matrix $\mA_t$ in Eqn. \ref{eqn:at_simplified}. More specifically,
\begin{equation}
    \bar{\rvx}_t = \mA_t\rvx_t = \exp(\kappa_1(t)) \Big[\rvx_t - (1 - \exp(\kappa_2(t)))\mP\rvx_t\Big], \label{eqn:theory_1}
\end{equation}
Since $\mP = \mH^\top(\mH\mH^\top)^{-1}\mH$ is an orthogonal projector, the matrix $\mI_d - \mP$ is also an orthogonal projector which projects any vector $\vv$ in the nullspace of $\mP$. Therefore, we can decompose the state $\rvx_t$ into two \emph{orthogonal} components $\rvx_t = \mP\rvx_t + (\mI_d - \mP)\rvx_t$. Plugging this form in Eqn. \ref{eqn:theory_1},
\begin{equation}
    \bar{\rvx}_t = \exp(\kappa_1(t)) \Big[(\mI_d - \mP)\rvx_t + \exp(\kappa_2(t))\mP\rvx_t\Big],
    \label{eqn:theory_2}
\end{equation}
Intuitively, near $t=T$ (i.e., at the start of reverse diffusion sampling), for a large static guidance weight $w$, $\exp(\kappa_2(t)) \rightarrow 0$. In this limit, from eqn. \ref{eqn:theory_2}, $\bar{\rvx}_t \approx (\mI_d - \mP)\rvx_t$. This implies that for a large guidance weight $w$, the diffusion dynamics are projected into the nullspace of the projection operator $\mP$. Intuitively, for an inverse problem like superresolution, this implies that near the start of the diffusion process, the projected diffusion dynamics correspond to the \emph{denoising of the high-frequency details} missing in $\mP\rvx_t$. This is because the projector operation, $\mP\rvx_t=\mH^\dag\mH\rvx_t$ can be interpreted as the pseudoinverse of the noisy degraded state $\rvx_t$, and, therefore, $(\mI_d - \mP)\rvx_t$ represents the high-frequency details missing from the signal component in $\mP\rvx_t$.

Moreover, near $t=0$ (i.e., near the end of reverse diffusion sampling), assuming the guidance weight $w$ is not too large, both coefficients $\exp(\kappa_1(t))$ and $\exp(\kappa_2(t)) \rightarrow 1$, which implies $\bar{\rvx}_t \approx \rvx_t$. This implies that near $t=0$, diffusion happens in the original space, which can prevent over-sharpening artifacts towards the end of sampling. Therefore, we hypothesize that a large $w$ can also lead to over-sharpened results near the end of sampling, resulting in artifacts in the generated samples. Therefore, introducing the projection $\mA_t$ as defined in Eqn. \ref{eqn:vp_conj_at}, introduces a tradeoff in the choice of $w$ to control for sample quality. Lastly, since the parameter $\lambda$ controls the \textit{magnitude} of $\bar{\rvx}_t$, it exhibits a similar tradeoff. Indeed, we will empirically demonstrate these tradeoffs in Section \ref{sec:exp_ablation}. While our discussion has been limited to diffusion models, a similar theoretical intuition also holds for flows (See Appendix~\ref{app:cci_flows} for proof).

\subsection{Extension to Noisy and Non-Linear Inverse Problems}
\label{sec:noisy_nl}

While our discussion has been primarily in the context of noiseless linear inverse problems, the conditional Conjugate Integrator framework can also be extended to develop samplers for noisy linear and non-linear inverse problems. We provide a more detailed explanation for the same in App. \ref{sec:app_noisy_nl}.

\section{Experiments}
\label{sec:experiment}
Next, we empirically demonstrate that our proposed samplers C-$\Pi$GDM/GFM outperform recent baselines on linear image restoration tasks regarding sampling speed vs. quality tradeoff. We then present ablation experiments highlighting the key parameters of our samplers. Lastly, we present design choices for solving noisy and non-linear inverse problems using our proposed framework.

\paragraph{Models and Dataset:}
For diffusion models, we utilize an unconditional pre-trained ImageNet~\citep{deng2009imagenet} checkpoint at 256$\times$256 resolution from OpenAI~\citep{dhariwal2021diffusion}\footnote{https://github.com/openai/guided-diffusion}. For evaluations on the FFHQ dataset~\cite{karras2019style}, we use a pre-trained checkpoint from \citet{choi2021ilvr} also at 256$\times$256 resolution. For flow model comparisons, we utilize three publicly available model checkpoints from \citet{liu2022flow}\footnote{https://github.com/gnobitab/RectifiedFlow}, trained on the AFHQ-Cat~\citep{choi2020stargan}, LSUN-Bedroom~\cite{yu15lsun}, and CelebA-HQ~\citep{karras2018progressive} datasets. Each flow model was trained at a pixel resolution of \(256 \times 256\). For diffusion models, we conduct evaluations on a 1k subset of the evaluation set. For flows, we conduct evaluations on the entire validation set.

\paragraph{Tasks and Metrics:}
We evaluate our samplers qualitatively (see Figure \ref{fig:qual}) and quantitatively on three challenging linear inverse problems under the noiseless setting. Firstly, we test \textbf{Image Super-Resolution}, enhancing images from bicubic-downsampled \(64 \times 64\) pixels to \(256 \times 256\) pixels. Secondly, we assess \textbf{Image Inpainting} performance on images with a fixed free-form center mask. Lastly, we evaluate our samplers on \textbf{Gaussian Deblurring}, applying a Gaussian kernel with $\sigma=3.0$ across a \(61 \times 61\) window. We evaluate the performance of each task based on three perceptual metrics: FID~\citep{heusel2017gans}, KID~\citep{binkowski2018demystifying} and LPIPS~\citep{zhang2018unreasonable}.

\paragraph{Methods and Baselines:} 
We assess the sample quality of our proposed C-$\Pi$GDM and C-$\Pi$GFM samplers using 5, 10, and 20 sampling steps (denoted as Number of Function Evaluations (NFE)). We conduct an extensive search to optimize the parameters \( w \), \(\lambda\) and \( \tau \) to identify the best-performing configuration based on sample quality.
For diffusion baselines, we include DDRM~\citep{kawar2022denoising}, DPS~\citep{chung2022diffusion}, and $\Pi\text{GDM}$~\citep{song2022pseudoinverse}. As recommended for DPS \citep{chung2022diffusion}, we use NFE=1000 for all tasks. For DDRM, we adhere to the original implementation and run it with $\eta_b=1.0$ and $\eta=0.85$ at NFE=20. We test our implementation of $\Pi\text{GDM}$ (see Section \ref{sec:cci_diffusion}), with NFE values of 5, 10, and 20 and use the recommended guidance schedule of $w_t=r_t^2$ across all tasks. For flow models, we consider the recently proposed method inspired by $\Pi$GDM running on OT-ODE path by \citet{pokle2024trainingfree} (which we refer to as $\Pi$GFM; see Appendix \ref{app:cci_flows}), and similarly run it with NFE values of 5, 10, and 20. We optimize all baselines by conducting an extensive grid search over $w$ and \(\tau\) for the best performance (in terms of sample quality).

\subsection{Quantitative Results}

\begin{table}[t]
  \centering
  \scriptsize
  \setlength{\tabcolsep}{2pt}
  \begin{tabular}{@{}c|c|cccc|cccc|cccc@{}}
    \toprule
    \multirow{2}{*}{\textbf{Flow Results}} & \multirow{2}{*}{NFE} & \multicolumn{4}{c|}{LPIPS$\downarrow$} & \multicolumn{4}{c|}{KID$\times 10^{-3}\downarrow$} & \multicolumn{4}{c}{FID$\downarrow$} \\
    \cmidrule(lr){3-6} \cmidrule(lr){7-10} \cmidrule(l){11-14}
    & & \multicolumn{2}{c}{\textbf{C-$\Pi$GFM}} & \multicolumn{2}{c|}{$\Pi$GFM} & \multicolumn{2}{c}{\textbf{C-$\Pi$GFM}} & \multicolumn{2}{c|}{$\Pi$GFM} & \multicolumn{2}{c}{\textbf{C-$\Pi$GFM}} & \multicolumn{2}{c}{$\Pi$GFM} \\
    \midrule
    \multirow{3}{*}{Inpainting} & 5 & \multicolumn{2}{c}{\textbf{0.125}} & \multicolumn{2}{c|}{0.240} & \multicolumn{2}{c}{\textbf{17.6}} & \multicolumn{2}{c|}{167.0} & \multicolumn{2}{c}{\textbf{26.95}} & \multicolumn{2}{c}{161.49}\\
    & 10 & \multicolumn{2}{c}{\textbf{0.074}} & \multicolumn{2}{c|}{0.188} & \multicolumn{2}{c}{\textbf{8.0}} & \multicolumn{2}{c|}{86.6} & \multicolumn{2}{c}{\textbf{14.64}} & \multicolumn{2}{c}{94.91}\\
    & 20 & \multicolumn{2}{c}{\textbf{0.065}} & \multicolumn{2}{c|}{0.144} & \multicolumn{2}{c}{\textbf{4.6}} & \multicolumn{2}{c|}{54.4} & \multicolumn{2}{c}{\textbf{10.93}} & \multicolumn{2}{c}{65.39}\\
    \midrule
    \multirow{3}{*}{Super-Resolution} & 5 & \multicolumn{2}{c}{\textbf{0.063}} & \multicolumn{2}{c|}{0.091} & \multicolumn{2}{c}{\textbf{5.5}} & \multicolumn{2}{c|}{17.5} & \multicolumn{2}{c}{\textbf{13.08}} & \multicolumn{2}{c}{21.84}\\
    & 10 & \multicolumn{2}{c}{\textbf{0.058}} & \multicolumn{2}{c|}{0.076} & \multicolumn{2}{c}{\textbf{3.6}} & \multicolumn{2}{c|}{12.2} & \multicolumn{2}{c}{\textbf{10.65}} & \multicolumn{2}{c}{16.73}\\
    & 20 & \multicolumn{2}{c}{\textbf{0.064}} & \multicolumn{2}{c|}{0.069} & \multicolumn{2}{c}{3.9} & \multicolumn{2}{c|}{\textbf{3.5}} & \multicolumn{2}{c}{11.07} & \multicolumn{2}{c}{\textbf{10.23}} \\
    \midrule
    \multirow{3}{*}{Deblurring} & 5 & \multicolumn{2}{c}{\textbf{0.083}} & \multicolumn{2}{c|}{0.114} & \multicolumn{2}{c}{\textbf{3.7}} & \multicolumn{2}{c|}{10.9} & \multicolumn{2}{c}{\textbf{12.86}} & \multicolumn{2}{c}{18.97} \\
    & 10 & \multicolumn{2}{c}{\textbf{0.077}} & \multicolumn{2}{c|}{0.088} & \multicolumn{2}{c}{\textbf{5.0}} & \multicolumn{2}{c|}{7.0} & \multicolumn{2}{c}{\textbf{14.41}} & \multicolumn{2}{c}{15.09} \\
    & 20 & \multicolumn{2}{c}{0.080} & \multicolumn{2}{c|}{\textbf{0.073}} & \multicolumn{2}{c}{7.9} & \multicolumn{2}{c|}{\textbf{3.1}} & \multicolumn{2}{c}{17.10} & \multicolumn{2}{c}{\textbf{11.35}} \\
    \bottomrule
    \toprule
    \textbf{Diffusion Results} & & \textbf{C-$\Pi$GDM} & $\Pi$GDM & DPS & DDRM & \textbf{C-$\Pi$GDM} & $\Pi$GDM & DPS & DDRM & \textbf{C-$\Pi$GDM} & $\Pi$GDM & DPS & DDRM \\
    \midrule
    \multirow{3}{*}{Super-Resolution} & 5 & \textbf{0.220} & 0.306 & \multirow{3}{*}{0.252} & \multirow{3}{*}{0.318} & \textbf{2.7} & 6.3 & \multirow{3}{*}{5.8} & \multirow{3}{*}{14.1} & \textbf{37.31} & 49.06 & \multirow{3}{*}{38.18} & \multirow{3}{*}{51.64}\\
    & 10 & \textbf{0.206} & 0.252 &  &  & \textbf{1.6} & 4.8 &  &  & \textbf{34.22} & 44.30 &  & \\
    & 20 & \textbf{0.207} & 0.222 &  &  & \textbf{1.7} & 2.5 &  &  & \textbf{34.28} & 37.36 &  & \\
    \midrule
    \multirow{3}{*}{Deblurring} & 5 & \textbf{0.272} & 0.349 & \multirow{3}{*}{0.619} & \multirow{3}{*}{0.336} & \textbf{3.89} & 14.1 & \multirow{3}{*}{59.5} & \multirow{3}{*}{12.3} & \textbf{44.42} & 63.94 & \multirow{3}{*}{139.58} & \multirow{3}{*}{62.53}\\
    & 10 & \textbf{0.272} & 0.294 &  &  & \textbf{3.6} & 5.3 &  &  & \textbf{43.37} & 47.80 &  & \\
    & 20 & 0.268 & \textbf{0.259} &  &  & \textbf{3.5} & 4.2 &  &  & \textbf{43.70} & 44.20 &  & \\
    \bottomrule
  \end{tabular}
  \caption{Comparison between Conjugate $\Pi$G(D/F)M and other baselines for noiseless linear inverse problems. Top: Flow models (CelebA-HQ) and Bottom: Diffusion Models (ImageNet). Entries in bold show the best performance for a given sampling budget.}
  \label{tab:quantative}
\end{table}

We present the results of our method applied to inverse problems in Table~\ref{tab:quantative}, specifically using the CelebA-HQ dataset for flow-based models and the ImageNet dataset for diffusion-based models. For a comprehensive review of additional results across different datasets, please refer to Appendix~\ref{app:add_results}. Our method consistently surpasses other approaches across all sampling budgets (indicated by NFE) for the inpainting task.
Similarly, our flow-based sampler (C-$\Pi$GFM) exhibits superior perceptual quality for image super-resolution at NFEs of 5 and 10. The $\Pi$GFM model only reaches comparable performance at higher NFEs. Remarkably, our diffusion-based sampler C-$\Pi$GDM outperforms all baselines across the entire range of NFEs. Notably, C-$\Pi$GDM outperforms competing baselines requiring 20-1000 NFEs in just 5 sampling steps on the challenging ImageNet dataset, demonstrating a significant speedup in sampling speed while preserving sample quality.  A similar pattern is observed in the image deblurring task, where the performance of $\Pi$GDM/$\Pi$GFM approaches that of our method only when the NFE is increased to 20 steps.

Interestingly, we observe a plateau in performance improvements at NFE=20 for both super-resolution and deblurring tasks using our method. This suggests that while our method efficiently utilizes the iterative model under a deterministic path with an Euler solver, further enhancements in performance, particularly at higher NFEs, might require integrating stochastic sampling techniques or more advanced solvers. This potential next step could unlock further gains from our approach in complex image processing tasks.

\subsection{Qualitative Results}

\begin{figure}[t]
\centering
\begin{adjustbox}{minipage=\linewidth,scale=1.0}
\begin{tabular}{@{}c@{\hspace{1pt}}c@{}}
\rotatebox[x=-10pt,y=1pt]{90}{Reference} & \begin{subfigure}{0.95\textwidth}
  \centering
  \includegraphics[width=0.19\linewidth]{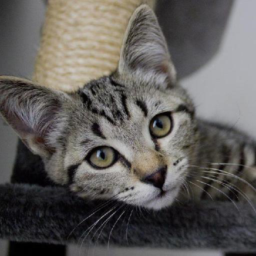}
  \includegraphics[width=0.19\linewidth]{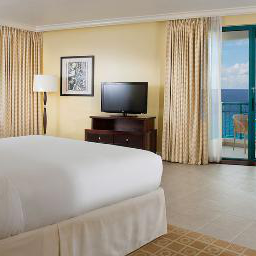}
  \includegraphics[width=0.19\linewidth]{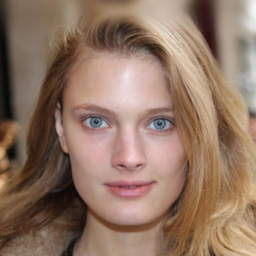}
  \includegraphics[width=0.19\linewidth]{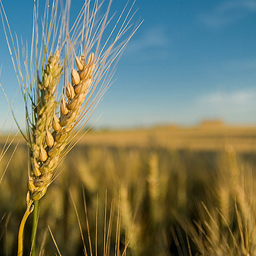}
  \includegraphics[width=0.19\linewidth]{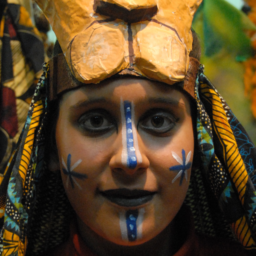}
\end{subfigure} \\
\rotatebox[x=-15pt,y=1pt]{90}{Distorted} & \begin{subfigure}{0.95\textwidth}
  \centering
  \includegraphics[width=0.19\linewidth]{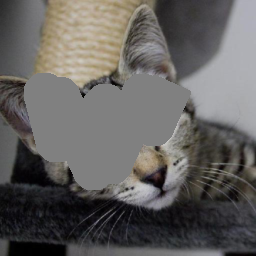}
  \includegraphics[width=0.19\linewidth]{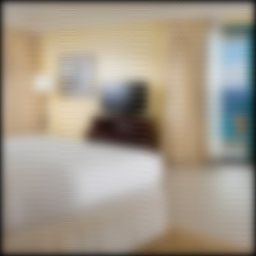}
  \includegraphics[width=0.19\linewidth]{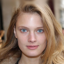}
  \includegraphics[width=0.19\linewidth]{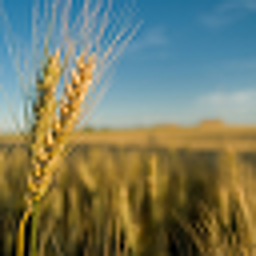}
  \includegraphics[width=0.19\linewidth]{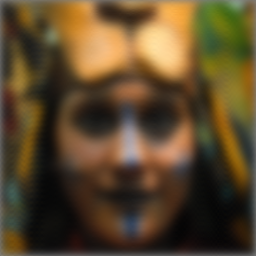}
\end{subfigure} \\
\rotatebox[x=-5pt,y=1pt]{90}{\textbf{C-$\Pi$G(D/F)M}} & \begin{subfigure}{0.95\textwidth}
  \centering
  \includegraphics[width=0.19\linewidth]{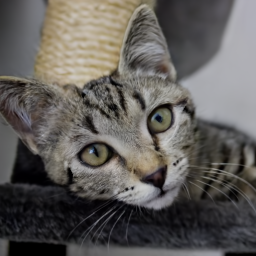}
  \includegraphics[width=0.19\linewidth]{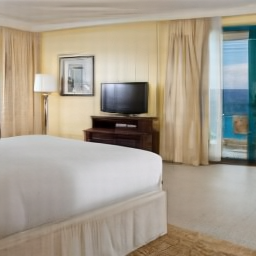}
  \includegraphics[width=0.19\linewidth]{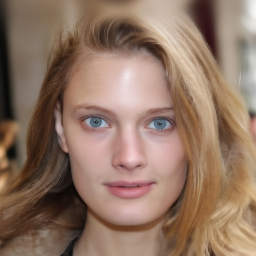}
  \includegraphics[width=0.19\linewidth]{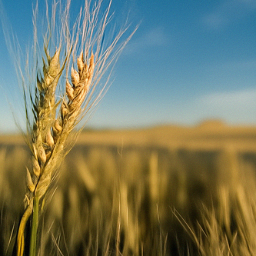}
  \includegraphics[width=0.19\linewidth]{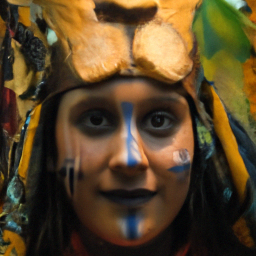}
\end{subfigure} \\
\rotatebox[x=-10pt,y=1pt]{90}{$\Pi$G(D/F)M} & \begin{subfigure}{0.95\textwidth}
  \centering
  \includegraphics[width=0.19\linewidth]{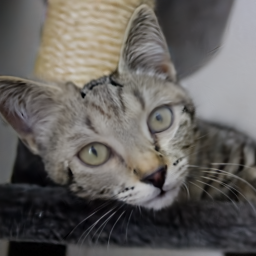}
  \includegraphics[width=0.19\linewidth]{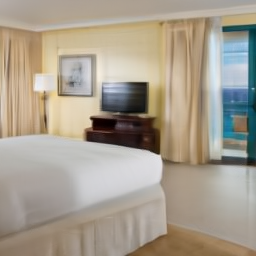}
  \includegraphics[width=0.19\linewidth]{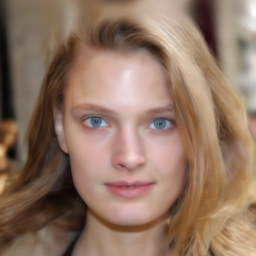}
  \includegraphics[width=0.19\linewidth]{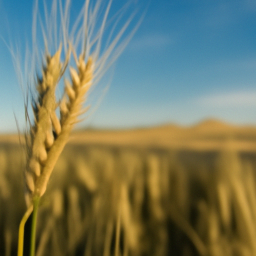}
  \includegraphics[width=0.19\linewidth]{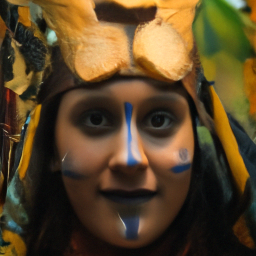}
\end{subfigure} \\
& \begin{subfigure}{0.19\textwidth}
  \centering
  \caption{AFHQ}
  \label{fig:afhq}
\end{subfigure}
\begin{subfigure}{0.19\textwidth}
  \centering
  \caption{LSUN}
  \label{fig:lsun}
\end{subfigure}
\begin{subfigure}{0.19\textwidth}
  \centering
  \caption{CelebA-HQ}
  \label{fig:celebahq}
\end{subfigure}
\begin{subfigure}{0.19\textwidth}
  \centering
  \caption{ImageNet}
  \label{fig:imagenet}
\end{subfigure}
\begin{subfigure}{0.19\textwidth}
  \centering
  \caption{FFHQ}
  \label{fig:ffhq}
\end{subfigure}
\end{tabular}
\end{adjustbox}
\caption{Qualitative comparison between C-$\Pi$G(D/F)M and $\Pi$G(D/F)M baselines on five different datasets. (\subref{fig:afhq}, \subref{fig:lsun}, \subref{fig:celebahq}) Inpainting, De-blurring, and 4x Super-resolution with C-$\Pi$GFM, respectively. (\subref{fig:imagenet},\subref{fig:ffhq}) 4x Image Super-resolution and De-blurring with C-$\Pi$GDM, respectively. ($\sigma_y=0$, NFE=5)}
\label{fig:qual}
\vspace{-1em}
\end{figure}

Figure~\ref{fig:qual} presents a qualitative comparison between our proposed method and the $\Pi$G(D/F)M baseline. The inpainting results in the first column reveal that $\Pi$GFM tends to introduce gray artifacts within the inpainted areas. This issue may stem from the initialization of the parameter $\tau$; optimal performance is achieved when $\tau \geq 0.2$, as established during our parameter tuning phase and corroborated by \citet{pokle2024trainingfree}. Consequently, insufficient NFE means $\Pi$GFM cannot effectively eliminate the artifacts associated with the inpainting mask in our experiments. For image super-resolution, our method excels in restoring fine details, particularly evident in high-frequency image components such as human hair and wheat ears. Similarly, for the deblurring task, our method qualitatively outperforms the baseline, especially in mitigating the over-smoothing artifacts (Figure \ref{fig:qual}, last column). Additional examples are provided in Appendix~\ref{app:qual}.

\subsection{Ablation Studies}
\label{sec:exp_ablation}

In this section, we further explore the impact of the hyperparameters \(w\), \(\lambda\), and \(\tau\), which were identified during our tuning phase and link to the theoretical insights discussed in Section~\ref{sec:theoretical_aspects}. We recognize that \(\tau\) is particularly task-specific and relatively straightforward to adjust. For instance, tasks such as inpainting require a smaller \(\tau\) to prevent masking artifacts, whereas tasks like super-resolution or deblurring benefit from a larger \(\tau\) to ensure effective initialization. Consequently, our discussion will primarily focus on the effects of \(w\) and \(\lambda\). Figure~\ref{fig:ablation} illustrates the impact of varying $w$ and $\lambda$ on sample quality for image super-resolution on the CelebA-HQ and ImageNet datasets.

From Figure \ref{fig:ablation} we make the following observations. Firstly, for both C-$\Pi$GDM and C-$\Pi$GFM samplers, we observe that the optimal value of $\lambda$ can differ from $\lambda=0$. This illustrates the usefulness of parameterizing $\mB_t$ in our sampler design. On the contrary, $\Pi$GDM or $\Pi$GFM samplers do not have this flexibility and, therefore, yield sub-optimal sample quality at different sampling budgets. Secondly, we observe that deviating from the optimal $\lambda$ can lead to degradation in sample quality. More specifically, we observed that deviating from our tuned value of $\lambda$ leads to either over-sharpening artifacts or blurry samples (See Figs. \ref{fig:app_fig_4}, \ref{fig:app_fig_7}). This is intuitive since $\lambda$ controls the scale of the transformation $\bar{\rvx}_t=\mA_t\rvx_t$ (see Eqn. \ref{eqn:theory_2}) and thus plays a significant role in conditioning the projected diffusion dynamics. We observe a similar tradeoff on varying the static guidance weight $w$ where a large magnitude of $w$ can lead to over-sharpened artifacts while a very small guidance weight can lead to blurry samples (See Figs. \ref{fig:app_fig_3}, \ref{fig:app_fig_6}). These empirical observations are consistent with our theoretical analysis in Section \ref{sec:theoretical_aspects}, confirming our theoretical intuition on the role of the sampler parameters $w$ and $\lambda$.

\begin{figure}[t]
\centering
\begin{subfigure}{0.24\textwidth}
  \centering
  \includegraphics[width=\linewidth]{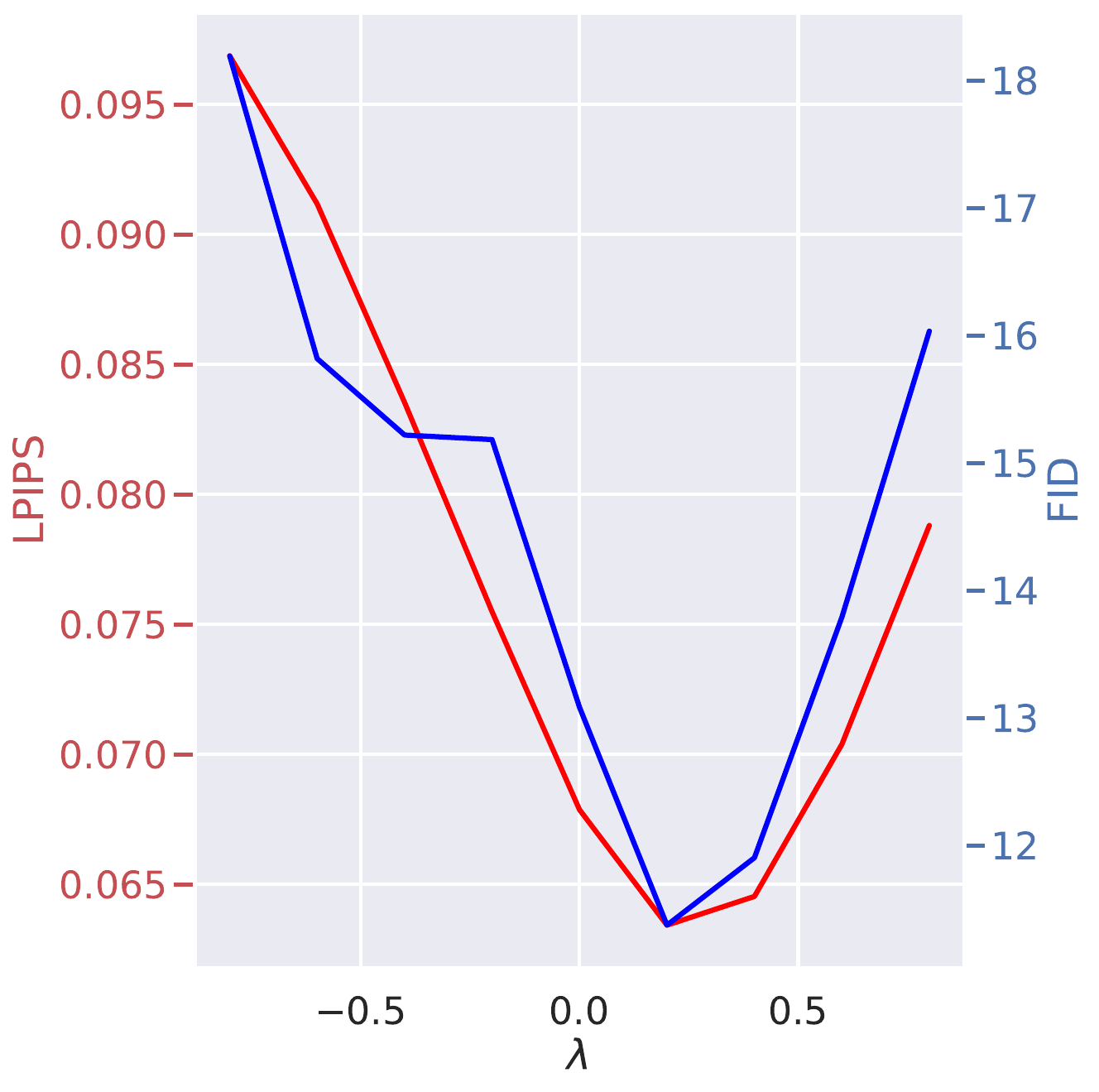}
  \caption{Flow $\lambda$}
  \label{fig:flow_lamb}
\end{subfigure}
\hfill
\begin{subfigure}{0.24\textwidth}
  \centering
  \includegraphics[width=\linewidth]{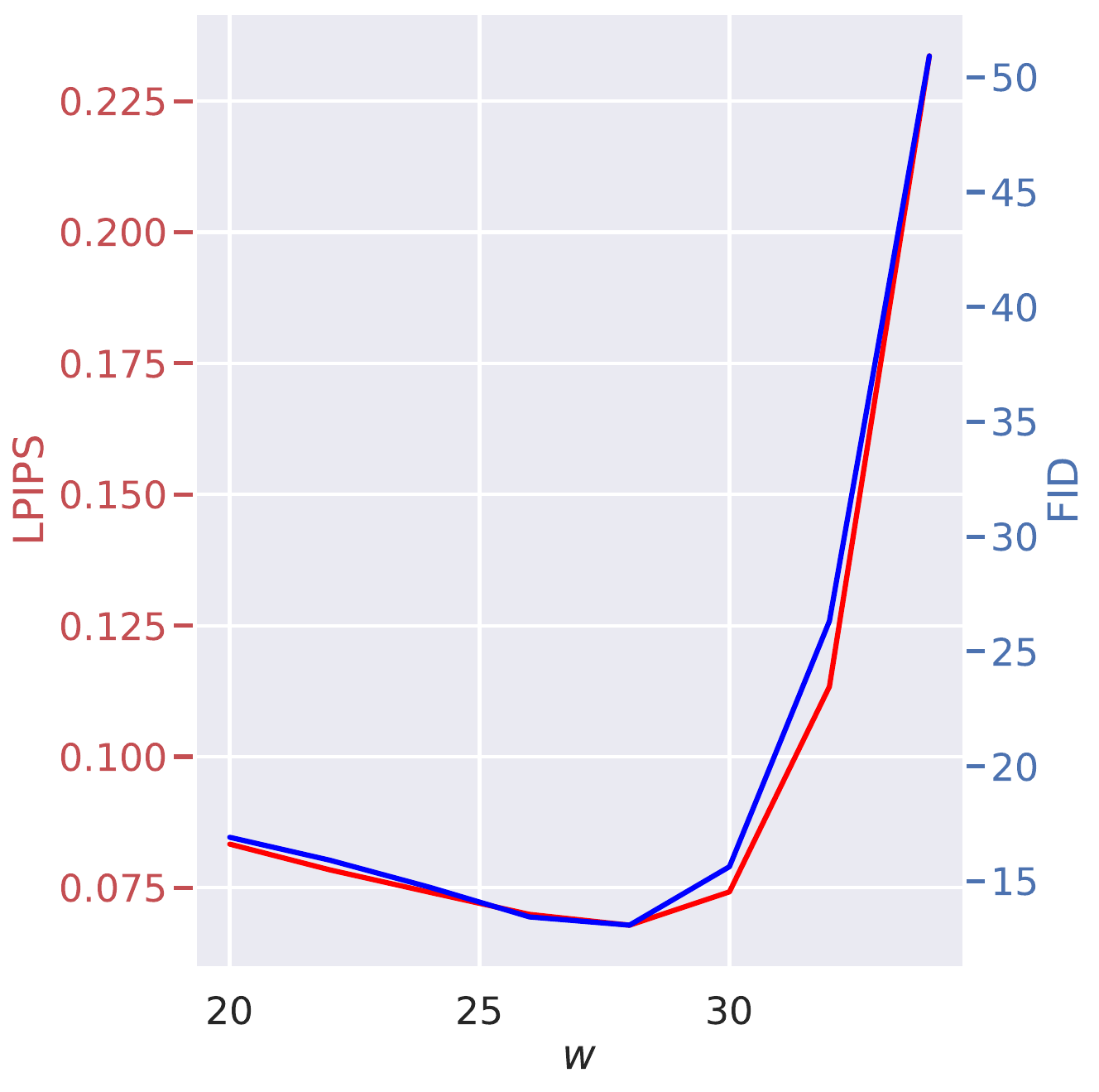}
  \caption{Flow $w$}
  \label{fig:flow_w}
\end{subfigure}
\hfill
\begin{subfigure}{0.24\textwidth}
  \centering
  \includegraphics[width=\linewidth]{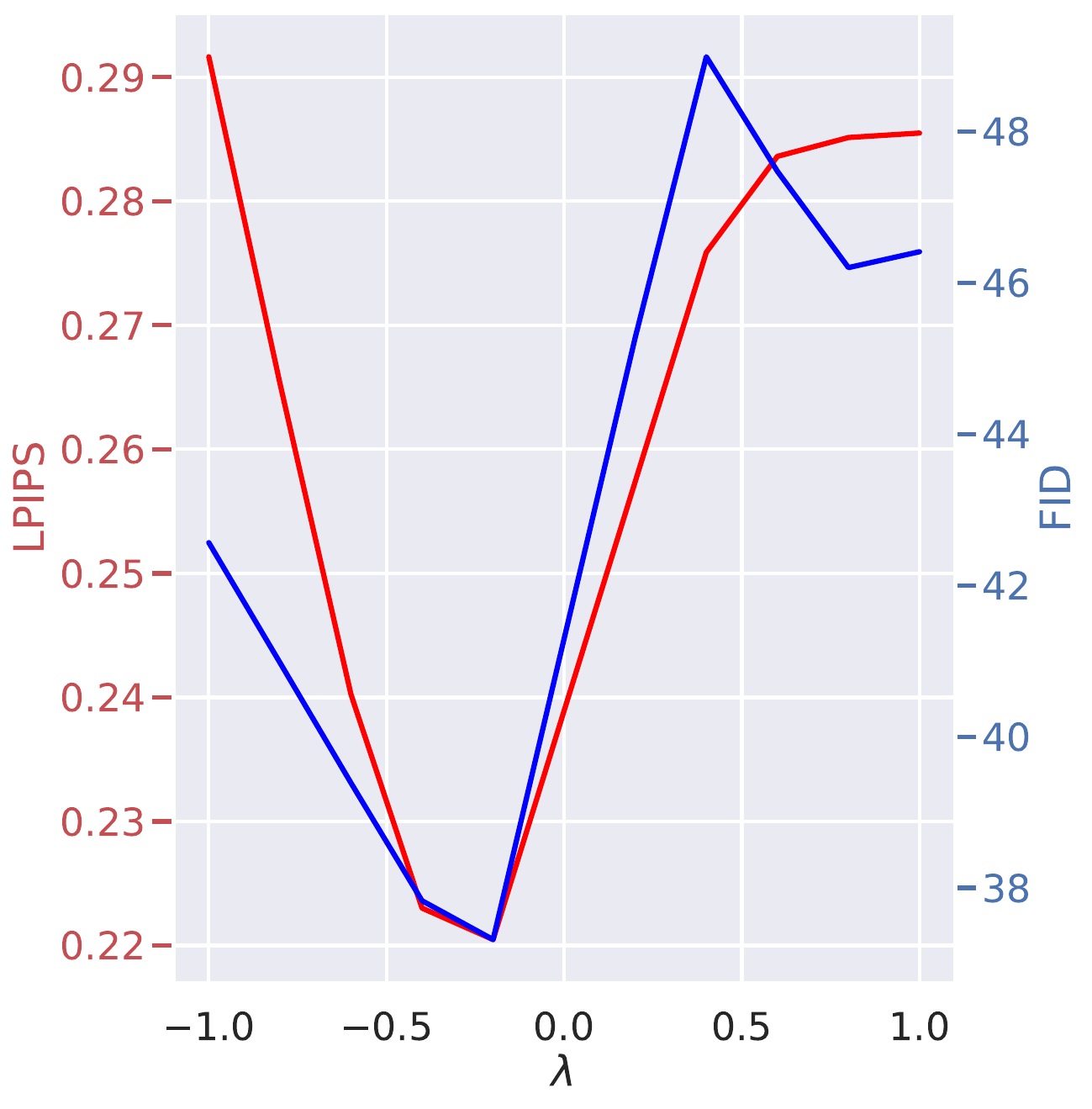}
  \caption{Diffusion $\lambda$}
  \label{fig:diff_lamb}
\end{subfigure}
\hfill
\begin{subfigure}{0.24\textwidth}
  \centering
  \includegraphics[width=\linewidth]{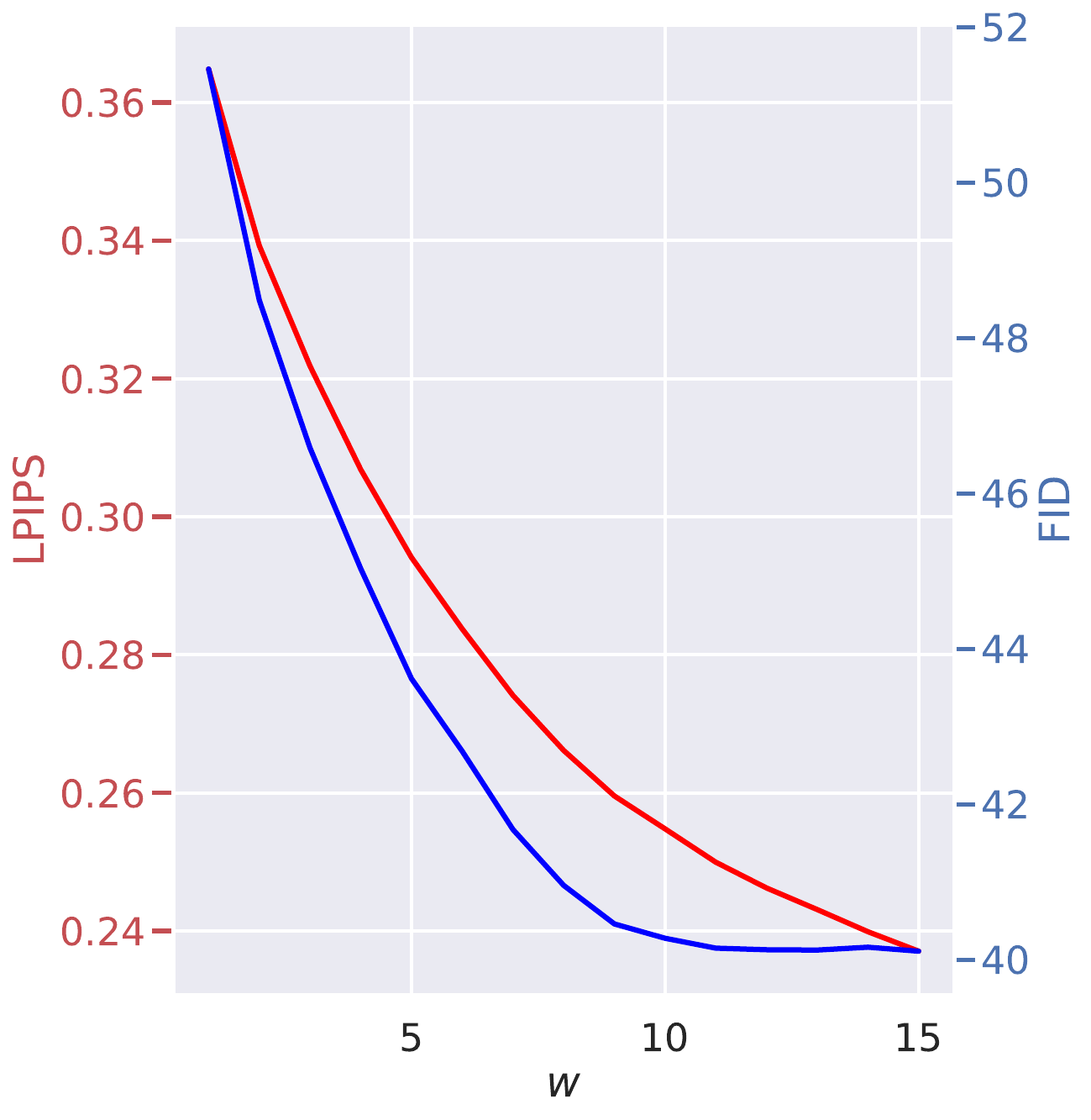}
  \caption{Diffusion $w$}
  \label{fig:diff_w}
\end{subfigure}
\caption{Impact of \(\lambda\) and \(w\) on sampling quality. \textcolor{red}{Red} curves and labels represent the LPIPS scores, while \textcolor{blue}{blue} curves and labels indicate the FID scores.}
\label{fig:ablation}
\vspace{-1em}
\end{figure}

\section{Related Works}

\paragraph{Fast Unconditional Sampling:} Recent research has significantly advanced the efficiency of the sampling process in unconditional diffusion/flow models~\citep{songscore, lipman2023flow, manduchi2024challenges}. One line of research involves designing efficient diffusion models to improve sampling by design \citep{karraselucidating,dockhornscore,pandey2023generative, song2023consistency}. Since our treatment of conditional Conjugate Integrators is quite generic, our method is readily compatible with most advancements in diffusion model design. Another line of work focuses on distilling a student model from a teacher model, enabling sampling in even a single step \citep{salimansprogressive,meng2022distillation,sauer2024fast}. However, since these methods require expensive re-training, there has been a significant interest in the development of fast samplers applicable to pretrained diffusion/flow models~\citep{liu2022pseudo, pandey2024efficient, shaul2023bespoke, zhang2023fast, lu2022dpm, songdenoising, gonzalez2024seeds}. Our work falls under the latter line of research, where we develop fast conditional samplers that can be applied to pretrained diffusion models. 

\paragraph{Conditional Iterative Refinement Models} have become prevalent for tasks requiring controlled generation. These models often involve training specialized conditional diffusion models~\citep{saharia2022palette,yang2024lossy,kong2020diffwave,pandeydiffusevae,preechakul2022diffusion,rombach2022high, podell2023sdxl, https://doi.org/10.48550/arxiv.2204.06125, Peebles_2023_ICCV, ma2024sit, esser2024scaling, chen2023pixartalpha} and may incorporate classifier-free guidance~\citep{ho2022classifier} or classifier guidance ~\citep{dhariwal2021diffusion,songscore} for conditional sampling. These approaches have also spurred research into solving inverse problems related to various image degradation transformations, such as inpainting and super-resolution~\citep{kawar2022denoising, chung2022diffusion, song2022pseudoinverse, mardani2023variational, pokle2024trainingfree}. Although these methods demonstrate promising outcomes, they are typically bottlenecked by a costly sampling process, emphasizing the need for a fast sampler to address inverse problems efficiently. Recent work~\cite{xu2024consistency} employs a consistency model~\cite{song2023consistency} to enhance posterior approximation, but incorporating an additional model may deviate from our proposal of using a single pre-trained model. DPM-Solver++ \citep{lu2023dpmsolver} also tackles the problem of accelerating guided sampling in diffusion models. However, unlike \citep{lu2023dpmsolver}, we incorporate the structure of the inverse problem in the sampler design.

\section{Discussion}
\label{sec:conclusion}
We present a generic framework for designing samplers for accelerating guided sampling in iterative refinement models. In this work, we explore a specific parameterization of this framework, which incorporates the structure of the inverse problem in sampler design. We provide a theoretical intuition behind our design choices and empirically justify its effectiveness in solving linear inverse problems in as few as 5 sampling steps compared to 20-1000 NFEs required by competing baselines. While our method can serve as an important step toward designing fast-guided samplers, there are several important future directions. Firstly, our parameterization of the transform $\mA_t$ can be more expressive by learning it directly during the sampling stage. Secondly, in this work, we consider inverse problems with a known degradation operator. Extending our framework for solving blind inverse problems could be an important research direction. Lastly, it would be interesting to adapt our solvers to techniques for solving inverse problems in latent diffusion models \citep{rout2024solving} to enhance sampling efficiency further.

\textbf{Broader Impact:} While our work has the potential to make synthetic data generation accessible, the techniques presented in this work should be used responsibly. Moreover, despite good sample quality in a limited sampling budget, restoration can sometimes lead to artifacts in the generated sample which can be undesirable in some domains like medical image analysis.

\begin{ack}
The authors thank Justus Will for providing valuable feedback on our manuscript. KP acknowledges support from the HPI Research Center in Machine Learning and Data Science at UC Irvine. SM acknowledges support from the National Science Foundation (NSF) under an NSF CAREER Award IIS-2047418 and IIS-2007719, the NSF LEAP Center, by the Department of Energy under grant DE-SC0022331, the IARPA WRIVA program, the Hasso Plattner Research Center at UCI, and by gifts from Qualcomm and Disney.
\end{ack}

\bibliography{main}
\bibliographystyle{unsrtnat}

\newpage
\appendix
\tableofcontents
\section{Proofs}
\label{app:proofs}

\subsection{Proof of Conditional flow dynamics}
\label{app:proof_1}
\begin{proof}
 Given a one-sided interpolant, $\rvx_t = \alpha_t \rvx_1 + \gamma_t \rvz, \quad \rvx_1 \sim p_\text{data}, \;\; \rvz \sim \mathcal{N}(0, \mI)$ satisfying regularity conditions as stated in \citep{albergo2023stochastic}, and a degraded signal $\rvy$ generated using Eqn. \ref{eqn:y_likelihood}, the conditional velocity field $\vb_(\rvx_t, \rvy, t)$ can be approximately estimated from the unconditional velocity field $\vb_\theta(\rvx_t, t)$ as \citep{pokle2024trainingfree},
 \begin{equation}
     \vb(\rvx_t, \rvy, t) \approx \vb_\theta(\rvx_t, t) + w_t\frac{\gamma_t}{\alpha_t}\Big[\gamma_t\dot{\alpha}_t - \dot{\gamma}_t\alpha_t\Big]\nabla_{\rvx_t} \log p(\rvy|\rvx_t)
 \end{equation}
 where $w_t$ represents a time-dependent scalar guidance schedule and $\dot{\sigma}_t$, $\dot{\alpha}_t$ represent the first-order time derivatives of $\sigma_t$ and $\alpha_t$, respectively.
    Our proof consists of two parts: Firstly, we establish the connection between the unconditional velocity field $\vb(\rvx_t, t)$ for the one-sided interpolant and the score function $\vs(\rvx_t, t)$ associated with the marginal distribution $p(\rvx_t)$. Secondly, we use this connection to estimate the conditional velocity $\vb(\rvx_t, \vy, t)$ in terms of $\vb(\rvx_t, t)$ to establish the required result.
    \newline\newline
    \textbf{Connection between $\vb(\rvx_t, t)$ and $\vs(\rvx_t, t)$:}
     For the one-sided interpolant, by definition,
     \begin{equation}
         \rvx_t = \alpha_t \rvx_1 + \gamma_t \rvz \quad \rvx_1 \sim p_\text{data}, \;\; \rvz \sim \mathcal{N}(0, \mI)
     \end{equation}
     Taking the expectation w.r.t $p(\rvx_1, \rvz)$ on both sides conditioned on the noisy state $\rvx_t$, we have,
     \begin{equation}
         \rvx_t = \alpha_t \mathbb{E}[\rvx_1|\rvx_t] + \gamma_t \mathbb{E}[\rvz|\rvx_t]
         \label{eqn:p1_1}
     \end{equation}
     Furthermore, we have the following result from \citet{albergo2023stochastic},
     \begin{equation}
         \mathbb{E}[\rvz|\rvx_t] = -\gamma_t \vs(\rvx_t, t) \label{eqn:p1_2}
     \end{equation}
     where $\vs(\rvx_t, t)$ represents the score function. From Eqns. \ref{eqn:p1_1}, \ref{eqn:p1_2}, it follows,
     \begin{equation}
         \mathbb{E}[\rvx_1|\rvx_t] = \frac{1}{\alpha_t}\Big[\rvx_t + \gamma_t^2 \vs(\rvx_t, t)\Big] \label{eqn:p1_3}
     \end{equation}
     Intuitively, the above result represents Tweedie's estimate \citep{stein1981estimation} for estimating $\hat{\rvx}_1= \mathbb{E}[\rvx_1|\rvx_t]$ in the context of one-sided stochastic interpolants. Next, the one-sided interpolant also induces an unconditional velocity field specified as:
     \begin{align}
         \rvb(\rvx_t, t) &= \dot{\alpha_t}\mathbb{E}[\rvx_1|\rvx_t] + \dot{\gamma_t}\mathbb{E}[\rvz|\rvx_t] \nonumber\\
         &= \dot{\alpha_t}\mathbb{E}[\rvx_1|\rvx_t] - \dot{\gamma_t}\gamma_t \vs(\rvx_t, t) \label{eqn:p1_4}
     \end{align}
      where $\dot{\gamma}_t$, $\dot{\alpha}_t$ represent the first-order time derivatives of $\gamma_t$ and $\alpha_t$, respectively. Substituting the result from Eqn. \ref{eqn:p1_3} into Eqn. \ref{eqn:p1_4}, we have the following result,
      \begin{equation}
          \vb(\rvx_t, t) = \frac{\dot{\alpha_t}}{\alpha_t} \rvx_t + \frac{\gamma_t}{\alpha_t} \Big[\gamma_t\dot{\alpha}_t - \dot{\gamma}_t\alpha_t\Big]\vs(\rvx_t, t) \label{eqn:p1_5}
      \end{equation}
      This concludes the first part of the proof.
      \newline\newline
      \textbf{Estimating the conditional velocity $\vb(\rvx_t, \vy, t)$ in terms of $\vb(\rvx_t, t)$.}
      For transport conditioned on $\rvy$, the conditional velocity can be expressed as (following the result in Eqn. \ref{eqn:p1_5}):
      \begin{align}
          \vb(\rvx_t, \vy, t) &= \frac{\dot{\alpha_t}}{\alpha_t} \rvx_t + \frac{\gamma_t}{\alpha_t} \Big[\gamma_t\dot{\alpha}_t - \dot{\gamma}_t\alpha_t\Big]\vs(\rvx_t, \vy, t) \\
          &=\frac{\dot{\alpha_t}}{\alpha_t} \rvx_t + \frac{\gamma_t}{\alpha_t} \Big[\gamma_t\dot{\alpha}_t - \dot{\gamma}_t\alpha_t\Big]\Big[\vs(\rvx_t, t) + w_t\nabla_{\rvx_t} \log p(\rvy|\rvx_t)\Big] \\
          &=\frac{\dot{\alpha_t}}{\alpha_t} \rvx_t + \frac{\gamma_t}{\alpha_t} \Big[\gamma_t\dot{\alpha}_t - \dot{\gamma}_t\alpha_t\Big]\vs(\rvx_t, t) + w_t\frac{\gamma_t}{\alpha_t} \Big[\gamma_t\dot{\alpha}_t - \dot{\gamma}_t\alpha_t\Big]\nabla_{\rvx_t} \log p(\rvy|\rvx_t) \\
          &= \vb(\rvx_t, t) + w_t\frac{\gamma_t}{\alpha_t} \Big[\gamma_t\dot{\alpha}_t - \dot{\gamma}_t\alpha_t\Big]\nabla_{\rvx_t} \log p(\rvy|\rvx_t)
      \end{align}
      Approximating the unconditional velocity $\vb(\rvx_t,t)$ using a parametric estimator $\vb_\theta(\rvx_t, t)$, we get the required result.
      \begin{equation}
          \vb(\rvx_t, \vy, t) \approx \vb_\theta(\rvx_t, t) + w_t\frac{\gamma_t}{\alpha_t} \Big[\gamma_t\dot{\alpha}_t - \dot{\gamma}_t\alpha_t\Big]\nabla_{\rvx_t} \log p(\rvy|\rvx_t)
      \end{equation}
 \end{proof}

 \subsection{Proof of Proposition \ref{prop:2}}
 We restate Proposition \ref{prop:2} for convenience,
 \label{app:proof_2}
 \begin{proposition*}
    \textit{For the conditional diffusion dynamics defined in Eqn. \ref{eqn:cond_diff}, introducing the transformation $\bar{\rvx}_t=\mA_t\rvx_t$ induces the following projected diffusion dynamics.
    \begin{equation}
    d\hat{\rvx}_t= \mA_t\mB_t\mA_t^{-1}\hat{\rvx}_t dt + d\bm{\Phi}_t\bm{\epsilon}_{\vtheta}\left(\rvx_t, t\right) - \frac{w_tr_t^{-2}}{2}\mG_t \mG_t^\top \frac{\partial \hat{\rvx}_1}{\partial \rvx_t}^\top (\mH^\dag\vy - \mP\hat{\rvx}_1) dt
\end{equation}
    \begin{equation}
     \mA_t = \exp{\left(\int_0^t \mB_s - \mF_s ds\right)}, \quad\quad \bm{\Phi}_t = -\int_0^t \frac{1}{2}\mA_s\mG_s\mG_s^\top\mC_\text{out}(s) ds,
 \end{equation}
 where $\mH^\dag = \mH^\top(\mH\mH^\top)^{-1}$ and $\mP=\mH^\top(\mH\mH^\top)^{-1}\mH$ represent the pseudoinverse and the orthogonal projector operators for the degradation operator $\mH$.}
\end{proposition*}
\begin{proof}
    We have the following form of the conditional diffusion dynamics
    \begin{align}
    \frac{d\rvx_t}{dt} &= \mF_t\rvx_t - \frac{1}{2}\mG_t \mG_t^\top \nabla_{\rvx_t} \log p(\rvx_t|\vy) \\
    &= \mF_t\rvx_t - \frac{1}{2}\mG_t \mG_t^\top \vs_\theta(\rvx_t, t) - \frac{1}{2}w_t\mG_t \mG_t^\top\nabla_{\rvx_t} \log p(\rvy|\rvx_t) \label{eqn:p2_1}
    \end{align}
    Given an affine transformation which projects the state $\rvx_t$ to $\hat{\rvx}_t$,
    \begin{equation}
        \hat{\rvx}_t = \mA_t\rvx_t
    \end{equation}
    Therefore, by the Chain Rule of calculus,
    \begin{align} 
    \frac{d \hat{\rvx}_t}{dt} &= \frac{d \mA_t}{d t} \rvx_t + \mA_t \frac{d \rvx_t}{d t} 
    \label{eqn:p2_2}
    \end{align}
    Substituting the ODE in Eqn. \ref{eqn:p2_1} in Eqn. \ref{eqn:p2_2}
    \begin{align}
        \frac{d \hat{\rvx}_t}{dt} &= \frac{d \mA_t}{d t} \rvx_t + \mA_t \Big[ \mF_t\rvx_t - \frac{1}{2}\mG_t \mG_t^\top \nabla_{\rvx_t} \vs_\theta(\rvx_t, t) - \frac{1}{2}w_t\mG_t \mG_t^\top\nabla_{\rvx_t} \log p(\rvy|\rvx_t)\Big] \\
        &= \Big[\frac{d \mA_t}{d t} + \mA_t\mF_t\Big]\rvx_t - \frac{1}{2}\mA_t\mG_t \mG_t^\top \vs_\theta(\rvx_t, t) - \frac{1}{2}w_t\mG_t \mG_t^\top\nabla_{\rvx_t} \log p(\rvy|\rvx_t)\Big]
    \end{align}
    Since, we have $\vs_\theta(\rvx_t, t) = \mC_\text{out}(t)\bm{\epsilon}_\theta(\rvx_t, t)$, the above equation can be simplified as,
    \begin{equation}
        \frac{d \hat{\rvx}_t}{dt} = \Big[\frac{d \mA_t}{d t} + \mA_t\mF_t\Big]\rvx_t - \frac{1}{2}\mA_t\mG_t \mG_t^\top \mC_\text{out}(t) \bm{\epsilon}_\theta(\rvx_t, t) - \frac{1}{2}w_t\mG_t \mG_t^\top\nabla_{\rvx_t} \log p(\rvy|\rvx_t)\Big]
    \end{equation}
    Further parameterizing,
    \begin{equation}
        \frac{d \mA_t}{d t} + \mA_t\mF_t = \mA_t\mB_t
    \end{equation}
    \begin{equation}
        \frac{d\bm{\Phi}_t}{dt} = -\frac{1}{2} \mA_t\mG_t\mG_t^\top\mC_\text{out}(t)
    \end{equation}
    which yields the required diffusion ODE in the projected space:
    \begin{equation}
         d\hat{\rvx}_t= \mA_t\mB_t\mA_t^{-1}\hat{\rvx}_t dt + d\bm{\Phi}_t\bm{\epsilon}_{\vtheta}\left(\rvx_t, t\right) - \frac{1}{2}w_t\mG_t \mG_t^\top\nabla_{\rvx_t} \log p(\rvy|\rvx_t)dt
     \end{equation}
     We have the following approximation for the conditional score $\nabla_{\rvx_t} \log p(\rvy|\rvx_t)$ (for the noiseless case $\sigma_y=0$)
     \begin{align}
    \nabla_{\rvx_t} \log p(\rvy|\rvx_t) &= r_t^{-2}\frac{\partial \hat{\rvx}_0}{\partial \rvx_t}^\top \mH^\top(\mH\mH^\top)^{-1} (\vy - \mH\hat{\rvx}_0) \\
    &= r_t^{-2}\frac{\partial \hat{\rvx}_0}{\partial \rvx_t}^\top (\mH^\top(\mH\mH^\top)^{-1}\vy - \mH^\top(\mH\mH^\top)^{-1}\mH\hat{\rvx}_0) \\
    &= r_t^{-2}\frac{\partial \hat{\rvx}_0}{\partial \rvx_t}^\top (\mH^\dag\vy - \mP\hat{\rvx}_0)
\end{align}
where $\mH^\dag = \mH^\top(\mH\mH^\top)^{-1}$ and $\mP=\mH^\top(\mH\mH^\top)^{-1}\mH$ represent the pseudoinverse and the orthogonal projector operators for the degradation operator $\mH$.
Substituting this form of the conditional score in projected diffusion dynamics, we have,
\begin{equation}
    d\hat{\rvx}_t= \mA_t\mB_t\mA_t^{-1}\hat{\rvx}_t dt + d\bm{\Phi}_t\bm{\epsilon}_{\vtheta}\left(\rvx_t, t\right) - \frac{w_tr_t^{-2}}{2}\mG_t \mG_t^\top \frac{\partial \hat{\rvx}_0}{\partial \rvx_t}^\top (\mH^\dag\vy - \mP\hat{\rvx}_0) dt
\end{equation}
This concludes the proof.
\end{proof}

\subsection{Proof of Proposition \ref{theorem:1}}
\label{app:proof_3}
We restate the theorem here for convenience.
\begin{proposition*}
    \textit{For a noiseless linear inverse problem with $\sigma_y=0$ and the conditional score approximated using Eqn. \ref{eqn:pigdm_approx}, introducing the transformation $\bar{\rvx}_t=\mA_t\rvx_t$ also induces the following projected diffusion dynamics.
    \begin{equation}
        d\bar{\rvx}_t = \mA_t\mB_t\mA_t^{-1}\bar{\rvx}_t + d\mPhi_y \vy + d\mPhi_s \bm{\epsilon}_\theta(\rvx_t, t) + d\mPhi_j \Big[\partial_{\rvx_t} \bm{\epsilon_\theta(\rvx_t, t)} (\mH^\dag\vy - \mP\hat{\rvx}_1)\Big]
    \end{equation}
     \begin{equation}
    \mA_t = \exp \Big[\int_0^t \mB_s - \Big(\mF_s + \frac{w_s r_s^{-2}}{2\mu_s^2}\mG_s \mG_s^\top\mP\Big)ds\Big] \qquad d\mPhi_y = -\frac{w_tr_t^{-2}}{2\mu_t}\mA_t\mG_t \mG_t^\top\mH^\dag
    \end{equation}
    \begin{equation}
        d\mPhi_s = -\frac{1}{2}\mA_t\mG_t \mG_t^\top\Big[\mI_d - \frac{w_tr_t^{-2}\sigma_t^2}{\mu_t^2}\mA_t\mP \Big] \mC_\text{out}(t) \qquad d\mPhi_j = -\frac{w_tr_t^{-2}\sigma_t^2}{2\mu_t}\mA_t\mG_t \mG_t^\top\mC_\text{out}(t)
    \end{equation}
    where $\exp(.)$ denotes the matrix exponential, $\mH^\dag$, and $\mP$ are the pseudoinverse and projector operators (as defined previously).}
\end{proposition*}
\begin{proof}
    The proof consists of two parts. Firstly, we simplify the conditional score $\nabla_{\rvx_t} \log p(\rvy|\rvx_t)$. Secondly, we plug the simplified form of the conditional score into the conditional diffusion dynamics and develop conjugate integrators.
    
    \textbf{Exploiting the linearity in $\nabla_{\rvx_t} \log p(\rvy|\rvx_t)$:} From the definition of the score $\nabla_{\rvx_t} \log p(\rvy|\rvx_t)$:
    \begin{equation}
        \nabla_{\rvx_t} \log p(\rvy|\rvx_t) = \frac{\partial \hat{\rvx}_0}{\partial \rvx_t}^\top \mH^\top \mSigma_t^{-1} (\vy - \mH\hat{\rvx}_0)
    \end{equation}
    where $\hat{\rvx}_0$ is the Tweedie's estimate of $\mathbb{E}(\rvx_1|\rvx_t)$ given by:
    \begin{equation}
        \hat{\rvx}_0 = \frac{1}{\mu_t}(\rvx_t + \sigma_t^2 \vs_\theta(\rvx_t, t))
    \end{equation}
    where $\mu_t$, $\sigma_t$ are the mean coefficient and standard deviation of the perturbation kernel $p(\rvx_t|\rvx_1) = \mathcal{N}(\mu_t \rvx_1 \sigma_t^2\mI_d)$, respectively, and $\mSigma_t = r_t^2(\mH\mH^\top)$ is the variance of the $\Pi$GDM approximation of $p(\rvy|\rvx_t)$ (for the noiseless case i.e. $\sigma_y=0$). Therefore,
\begin{align}
    \nabla_{\rvx_t} \log p(\rvy|\rvx_t) &= \frac{\partial \hat{\rvx}_0}{\partial \rvx_t}^\top \mH^\top \mSigma_t^{-1} (\vy - \mH\hat{\rvx}_0) \\
    &= \frac{1}{\mu_t}(\mI_d + \sigma_t^2\underbrace{\nabla_{\rvx_t} \vs_\theta(\rvx_t, t)}_{=\mS_\theta(\rvx_t, t)}) \mH^\top \mSigma_t^{-1} (\vy - \mH\hat{\rvx}_0) \\
    &= \frac{1}{\mu_t}\Big[\mH^\top \mSigma_t^{-1} (\vy - \mH\hat{\rvx}_0) + \sigma_t^2\mS_\theta(\rvx_t, t) \mH^\top \mSigma_t^{-1} (\vy - \mH\hat{\rvx}_0)\Big] \\
    &= \frac{1}{\mu_t}\Big[\mH^\top \mSigma_t^{-1} (\vy - \frac{1}{\mu_t}\mH(\rvx_t + \sigma_t^2 \vs_\theta(\rvx_t, t))) + \sigma_t^2\mS_\theta(\rvx_t, t) \mH^\top \mSigma_t^{-1} (\vy - \mH\hat{\rvx}_0)\Big] \\
    &= \underbrace{\frac{1}{\mu_t}\mH^\top \mSigma_t^{-1}\vy - \frac{1}{\mu_t^2}\mH^\top \mSigma_t^{-1}\mH\rvx_t}_{\text{Linear Terms}} \\
    &\qquad\qquad\underbrace{- \frac{\sigma_t^2}{\mu_t^2}\mH^\top \mSigma_t^{-1}\mH \vs_\theta(\rvx_t, t)) + \frac{\sigma_t^2}{\mu_t}\mS_\theta(\rvx_t, t) \mH^\top \mSigma_t^{-1} (\vy - \mH\hat{\rvx}_0)}_{\text{Non-Linear Terms}} \label{eqn:p3_1}
\end{align}
where $\mS_\theta$ denotes the second-order derivative of the score function $\vs_\theta(\rvx_t, t)$. Therefore, the conditional score $\nabla_{\rvx_t} \log p(\rvy|\rvx_t)$, can be decomposed into a combination of linear and non-linear terms. Next, we use this decomposition to design conjugate integrators for noiseless linear inverse problems.

\textbf{Conjugate Integrator Design:} From Eqn. \ref{eqn:p2_1}, the conditional reverse diffusion dynamics can be specified as:
    \begin{equation}
    \frac{d\rvx_t}{dt} = \mF_t\rvx_t - \frac{1}{2}\mG_t \mG_t^\top \vs_\theta(\rvx_t, t) - \frac{1}{2}w_t\mG_t \mG_t^\top\nabla_{\rvx_t} \log p(\rvy|\rvx_t)
    \end{equation}
    Plugging in the form of the conditional score in Eqn. \ref{eqn:p3_1} in the above equation, we have,
    \begin{align}
        \frac{d\rvx_t}{dt} &= \mF_t\rvx_t - \frac{1}{2}\mG_t \mG_t^\top \vs_\theta(\rvx_t, t) - \frac{1}{2}w_t\mG_t \mG_t^\top\nabla_{\rvx_t} \log p(\rvy|\rvx_t) \\
        &= \mF_t\rvx_t - \frac{1}{2}\mG_t \mG_t^\top \vs_\theta(\rvx_t, t) - \frac{1}{2}w_t\mG_t \mG_t^\top\Big[\frac{1}{\mu_t}\mH^\top \mSigma_t^{-1}\vy - \frac{1}{\mu_t^2}\mH^\top \mSigma_t^{-1}\mH\rvx_t
    \\&\qquad- \frac{\sigma_t^2}{\mu_t^2}\mH^\top \mSigma_t^{-1}\mH \vs_\theta(\rvx_t, t)) + \frac{\sigma_t^2}{\mu_t}\mS_\theta(\rvx_t, t) \mH^\top \mSigma_t^{-1} (\vy - \mH\hat{\rvx}_0)\Big] \\
     &= \Big[\mF_t + \frac{w_t}{2\mu_t^2}\mG_t \mG_t^\top\mH^\top \mSigma_t^{-1}\mH\Big]\rvx_t - \frac{w_t}{2\mu_t}\mG_t \mG_t^\top\mH^\top \mSigma_t^{-1}\vy \\&\qquad - \frac{1}{2}\mG_t \mG_t^\top\Big[\mI_d - \frac{w_t\sigma_t^2}{\mu_t^2}\mH^\top \mSigma_t^{-1}\mH \Big] \vs_\theta(\rvx_t, t) - \frac{w_t\sigma_t^2}{2\mu_t}\mG_t \mG_t^\top\Big[\mS_\theta(\rvx_t, t) \mH^\top \mSigma_t^{-1} (\vy - \mH\hat{\rvx}_0)\Big]
    \end{align}
    Given an affine transformation which projects the state $\rvx_t$ to $\hat{\rvx}_t$,
    \begin{equation}
        \hat{\rvx}_t = \mA_t\rvx_t
    \end{equation}
    the projected diffusion dynamics can be specified as:
    \begin{align}
        \frac{d\bar{\rvx}_t}{dt} &= \Big[\frac{d\mA_t}{dt} + \mA_t(\mF_t + \frac{w_t}{2\mu_t^2}\mG_t \mG_t^\top\mH^\top \mSigma_t^{-1}\mH)\Big]\rvx_t - \frac{w_t}{2\mu_t}\mA_t\mG_t \mG_t^\top\mH^\top \mSigma_t^{-1}\vy \\&\qquad - \frac{1}{2}\mA_t\mG_t \mG_t^\top\Big[\mI_d - \frac{w_t\sigma_t^2}{\mu_t^2}\mA_t\mH^\top \mSigma_t^{-1}\mH \Big] \vs_\theta(\rvx_t, t) - \frac{w_t\sigma_t^2}{2\mu_t}\mA_t\mG_t \mG_t^\top\Big[\mS_\theta(\rvx_t, t) \mH^\top \mSigma_t^{-1} (\vy - \mH\hat{\rvx}_1)\Big]
    \end{align}
    Furthermore, the score network is parameterized as $\vs_\theta(\rvx_t, t)=\mC_\text{out}(t)\bm{\epsilon}_\theta(\rvx_t, t)$. Consequently, $\mS_\theta(\rvx_t, t)=\mC_\text{out}(t)\partial_t \bm{\epsilon}_\theta(\rvx_t, t)$. Lastly, we reparameterize $\mSigma_t = r_t^2 \mSigma$ where $\mSigma = \mH\mH^\top$. Plugging these parameterizations in the projected diffusion dynamics, we have,
    \begin{align}
        \frac{d\bar{\rvx}_t}{dt} &= \Big[\frac{d\mA_t}{dt} + \mA_t(\mF_t + \frac{w_tr_t^{-2}}{2\mu_t^2}\mG_t \mG_t^\top\mH^\top \mSigma^{-1}\mH)\Big]\rvx_t - \frac{w_tr_t^{-2}}{2\mu_t}\mA_t\mG_t \mG_t^\top\mH^\top \mSigma^{-1}\vy \\&\qquad - \frac{1}{2}\mA_t\mG_t \mG_t^\top\Big[\mI_d - \frac{w_t r_t^{-2}\sigma_t^2}{\mu_t^2}\mA_t\mH^\top \mSigma^{-1}\mH \Big] \mC_\text{out}(t)\bm{\epsilon}_\theta(\rvx_t, t) \\
        &\qquad-\frac{w_t r_t^{-2} \sigma_t^2}{2\mu_t}\mA_t\mG_t \mG_t^\top\mC_\text{out}(t)\Big[\partial_{\rvx_t} \bm{\epsilon_\theta(\rvx_t, t)} \mH^\top \mSigma^{-1} (\vy - \mH\hat{\rvx}_0)\Big]
    \end{align}
    We then parameterize,
    \begin{equation}
        \frac{d\mA_t}{dt} + \mA_t\Big(\mF_t + \frac{w_t r_t^{-2}}{2\mu_t^2}\mG_t \mG_t^\top\mH^\top \mSigma_t^{-1}\mH\Big) = \mA_t\mB_t
    \end{equation}
    This implies,
    \begin{equation}
    \mA_t = \exp \Big[\int_0^t \mB_s - \Big(\mF_s + \frac{w_s r_s^{-2}}{2\mu_s^2}\mG_s \mG_s^\top\mP\Big)ds\Big]
    \end{equation}
    where $\exp(.)$ denotes the matrix exponential.
    Furthermore, we parameterize,
    \begin{equation}
        \mPhi_y = -\int_0^t \frac{w_sr_s^{-2}}{2\mu_s}\mA_s\mG_s \mG_s^\top\mH^\dag ds
        \label{eqn:phi_y}
    \end{equation}
    \begin{equation}
        \mPhi_s = - \int_0^t \frac{1}{2}\mA_s\mG_s \mG_s^\top\Big[\mI_d - \frac{w_sr_s^{-2}\sigma_s^2}{\mu_s^2}\mA_s\mP \Big] \mC_\text{out}(s) ds
        \label{eqn:phi_s}
    \end{equation}
    \begin{equation}
        \mPhi_j = -\int_0^t \frac{w_sr_s^{-2}\sigma_s^2}{2\mu_s}\mA_s\mG_s \mG_s^\top\mC_\text{out}(s) ds
        \label{eqn:phi_j}
    \end{equation}
    With this parameterization, the projected diffusion dynamics can be compactly specified as follows:
    \begin{equation}
        d\bar{\rvx}_t = \mA_t\mB_t\mA_t^{-1}\bar{\rvx}_t + d\mPhi_y \vy + d\mPhi_s \bm{\epsilon}_\theta(\rvx_t, t) + d\mPhi_j \Big[\partial_{\rvx_t} \bm{\epsilon_\theta(\rvx_t, t)}  (\mH^\dag\vy - \mP\hat{\rvx}_0)\Big]
    \end{equation}
    This concludes the proof.
\end{proof}

\subsection{Simplification in Eqn. \ref{eqn:at_simplified}}
\label{app:proof_4}
We restate the result for convenience. The matrix exponential in Eqn. \ref{eqn:vp_conj_at}
\begin{equation}
    \mA_t = \exp \Big[\int_0^t \Big(\lambda + \frac{1}{2}\beta_s\Big) ds \,\mI_d -\frac{w}{2}\Big(\int_0^t \beta_s ds\Big)\mP\Big]
\end{equation}
can be simplified as,
\begin{equation}
    \mA_t = \exp(\kappa_t^1) \Big[\mI_d + (\exp(\kappa_t^2) - 1)\mP\Big], \quad\kappa_t^1 = \int_0^t \Big(\lambda + \frac{1}{2}\beta_s\Big) ds,\quad\kappa_t^2=-\frac{w}{2}\int_0^t \beta_s ds
\end{equation}
where $\mP$ is the orthogonal projector corresponding to the degradation operator $\mH$.
\begin{proof}
    We have,
    \begin{align}
        \mA_t = \exp (\kappa_t^1\mI_d + \kappa_t^2\mP) = \exp (\kappa_t^1 \,\mI_d) \exp(\kappa_t^2\mP)
    \end{align}
    The above result follows since $\mI_d \mP = \mP \mI_d$ (commutative under multiplication). Moreover, we can further simplify the matrix exponential in $\mP$ as follows,
    \begin{align}
        \exp(\kappa_t^2\mP) &= \sum_{i=0}^\infty \frac{(\kappa_t^2)^i \mP^i }{i!} \\
        &= \mI_d + \sum_{i=1}^\infty \frac{(\kappa_t^2)^i \mP^i }{i!} = \mI_d + \Big[\sum_{i=1}^\infty \frac{(\kappa_t^2)^i}{i!}\Big]\mP
    \end{align}
    The above result follows from the property of orthogonal projectors $P^2=P$. Therefore,
    \begin{align}
        \exp(\kappa_t^2\mP) &= \mI_d + \Big[\sum_{i=1}^\infty \frac{(\kappa_t^2)^i}{i!}\Big]\mP = \mI_d + \Big[\sum_{i=0}^\infty \frac{(\kappa_t^2)^i}{i!} - 1\Big]\mP \\
        \exp(\kappa_t^2\mP) &= \mI_d + \Big[\exp(\kappa_t^2) - 1\Big]\mP
    \end{align}
    which concludes the proof.
\end{proof}

\subsection{Proof of Proposition for Solving Noisy Inverse Problems}
\label{app:proof_5}
We restate the result here for convenience.
\begin{proposition*}
    \textit{For the noisy inverse problem,
    \begin{equation}
        \vy = \mH\rvx_0 + \sigma_y \rvz, \quad \rvz \sim \gN(0, \mI_d),
    \end{equation}
    given the transformation $\mA_t$ for the noiseless case as defined in Eqn. \ref{eqn:at_simplified}, the corresponding \emph{noisy} transformation can be approximated as,
    \begin{equation}
        \mA_t^{\sigma_y} = \mA_t + \kappa_3(t) \mH^\dag (\mH^\dag)^\top + \mathcal{O}(\sigma_y^4)
    \end{equation}
    \begin{equation}
        \kappa_3(t) = \frac{w\sigma_y^2}{2}\Big(\int_0^t \frac{\beta_s}{r_s^2}ds\Big) \Big[\exp \Big(\kappa_1(t) + \kappa_2(t)\Big) - 1\Big]
    \end{equation}
    Consequently, the inverse of the transformation $\mA_t^{\sigma_y}$ can be approximated as,}
    \begin{equation}
        (\mA_t^{\sigma_y})^{-1} \approx \mA_t^{-1} - \kappa_3(t)\mA_t^{-1}\mH^\dag(\mH^\dag)^\top\mA_t^{-1} + \mathcal{O}(\sigma_y^4)
    \end{equation}
\end{proposition*}
\begin{proof}
We have,
\begin{equation}
    \mA_t^{\sigma_y} = \exp \Big[\int_0^t \Big(\lambda + \frac{1}{2}\beta_s\Big) ds -\frac{w}{2}\Big(\int_0^t \beta_s \mH^\top(\mH\mH^\top + \frac{\sigma_y^2}{r_t^2}\mI_d)^{-1}\mH ds\Big)\Big]
\end{equation}
From perturbation analysis, we introduce the following first-order approximation,
\begin{align}
    \mH^\top(\mH\mH^\top + \frac{\sigma_y^2}{r_t^2}\mI_d)^{-1}\mH &\approx \mH^\top\Big[(\mH\mH^\top)^{-1} - \frac{\sigma_y^2}{r_t^2}(\mH\mH^\top)^{-2}\Big]\mH + \mathcal{O}(\sigma_y^4) \\
    &\approx \mH^\top(\mH\mH^\top)^{-1}\mH - \frac{\sigma_y^2}{r_t^2}\mH^\top(\mH\mH^\top)^{-2}\mH + \mathcal{O}(\sigma_y^4) \\
    &\approx \mP - \frac{\sigma_y^2}{r_t^2} \mH^\dag(\mH^\dag)^\top
\end{align}
Substituting this approximation in the expression for $\mA_t^{\sigma_y}$ (and ignoring terms in $\mathcal{O}(\sigma_y^4)$),
\begin{align}
     \mA_t^{\sigma_y} &\approx \exp \Big[\int_0^t \Big(\lambda + \frac{1}{2}\beta_s\Big) ds \mI_d - \frac{w}{2}\Big(\int_0^t \beta_s \Big[\mP - \frac{\sigma_y^2}{r_s^2} \mH^\dag(\mH^\dag)^\top\Big] ds\Big)\Big] \\
     &= \exp \Big[\underbrace{\int_0^t \Big(\lambda + \frac{1}{2}\beta_s\Big) ds}_{=\kappa_1(t)} \mI_d \underbrace{- \frac{w}{2}\Big(\int_0^t \beta_s ds\Big)}_{=\kappa_2(t)} \mP + \frac{w\sigma_y^2}{2}\Big(\int_0^t \frac{\beta_s}{r_s^2}ds\Big) \mH^\dag(\mH^\dag)^\top \Big] \\
     &= \exp \Big[\kappa_1(t)\mI_d + \kappa_2(t) \mP + \frac{w\sigma_y^2}{2}\Big(\int_0^t \frac{\beta_s}{r_s^2}ds\Big) \mH^\dag(\mH^\dag)^\top \Big] \label{eqn:p4_1}
\end{align}
From the definition of the matrix exponential, it can be shown that the $\mA_t^{\sigma_y}$ in Eqn. \ref{eqn:p4_1} can be approximated as:
\begin{align}
    \mA_t^{\sigma_y} \approx \exp \Big[\kappa_1(t)\mI_d + \kappa_2(t) \mP\Big] + \underbrace{\frac{w\sigma_y^2}{2}\Big(\int_0^t \frac{\beta_s}{r_s^2}ds\Big) \Big[\exp \Big(\kappa_1(t) + \kappa_2(t)\Big) - 1\Big]}_{=\kappa_3(t)}\mH^\dag(\mH^\dag)^\top + \mathcal{O}(\sigma_y^4)
\end{align}
Ignoring the higher-order terms, we have,
\begin{equation}
    \mA_t^{\sigma_y} \approx \mA_t + \kappa_3(t)\mH^\dag(\mH^\dag)^\top
\end{equation}
Consequently, we can also approximate the inverse of $\mA_t^{\sigma_y}$, as follows,
\begin{align}
    (\mA_t^{\sigma_y})^{-1} &= [\mA_t + \kappa_3(t)\mH^\dag(\mH^\dag)^\top]^{-1} \\
    & \approx \mA_t^{-1} - \kappa_3(t)\mA_t^{-1}\mH^\dag(\mH^\dag)^\top\mA_t^{-1} + \mathcal{O}(\sigma_y^4)
\end{align}
which concludes the proof.
\end{proof}

\section{Conditional Conjugate Integrators: Flows}
\label{app:cci_flows}

\subsection{Background}
This section discusses conditional conjugate integrators in the context of flows. For brevity, we skip deriving our results for flows since the derivations can be similar to the analysis of diffusion models with minor parameterization changes. Recall that the conditional dynamics for flows are specified as follows \citep{pokle2024trainingfree}:
\begin{equation}
     \vb(\rvx_t, \rvy, t) \approx \vb_\theta(\rvx_t, t) + w_t\frac{\gamma_t}{\alpha_t}\Big[\gamma_t\dot{\alpha}_t - \dot{\gamma}_t\alpha_t\Big]\nabla_{\rvx_t} \log p(\rvy|\rvx_t)
 \end{equation}
where $\vb_\theta(\rvx_t, t)$ represents the pre-trained velocity field for a flow.
Moreover, we restate the form of the conditional score $\nabla_{\rvx_t} \log p(\rvy|\rvx_t)$ for convenience.
\begin{equation}
    \nabla_{\rvx_t} \log p(\rvy|\rvx_t) = \frac{\partial \hat{\rvx}_1}{\partial \rvx_t}^\top \mH^\top(r_t^2\mH\mH^\top + \sigma_y^2 \mI_d)^{-1} (\vy - \mH\hat{\rvx}_1)
\end{equation}
 where $\hat{\rvx}_1$ represents the Tweedie's estimate of the first moment of $\mathbb{E}(\rvx_t|\rvx_1)$,
\begin{equation}
     \hat{\rvx}_1 = \mathbb{E}[\rvx_1|\rvx_t] = \frac{1}{\alpha_t}\Big[\rvx_t + \gamma_t^2 \vs(\rvx_t, t)\Big] 
\end{equation}
where $\vs(\rvx_t, t)$ represents the score function associated with the marginal distribution $p(\rvx_t)$. It can be shown that $\hat{\rvx}_1$ can also be expressed in terms of the pre-trained velocity field $\vb_\theta(\rvx_t, t)$ as follows,
\begin{equation}
    \hat{\rvx}_1 = \frac{1}{\gamma_t\dot{\alpha}_t - \dot{\gamma}_t\alpha_t}\Big[-\dot{\gamma}_t\rvx_t + \gamma_t \vb_\theta(\rvx_t, t)\Big]
    \label{eqn:tweedie_flow}
\end{equation}

\subsection{Conditional Conjugate Integrators for Flows}
Analogous to diffusion models, we can design conditional conjugate samplers for flows that treat the conditional score $\nabla_{\rvx_t} \log p(\rvy|\rvx_t)$ as a black box. Similar to Proposition \ref{prop:2}, by introducing the transformation $\bar{\rvx}_t=\mA_t\rvx_t$, we have the projected flow dynamics,
\begin{equation}
    d\hat{\rvx}_t= \mA_t\mB_t\mA_t^{-1}\hat{\rvx}_t dt + d\bm{\Phi}_t\vb_{\vtheta}\left(\rvx_t, t\right) + w_tr_t^{-2}\frac{\partial \hat{\rvx}_1}{\partial \rvx_t}^\top (\mH^\dag\vy - \mP\hat{\rvx}_1) dt
    \label{eqn:c_pigfm}
\end{equation}
\begin{equation}
     \mA_t = \bm{\exp}{\left(\int_0^t \mB_s ds\right)}, \quad\quad \bm{\Phi}_t = \int_0^t \mA_s ds,
 \end{equation}
 where $\mH^\dag = \mH^\top(\mH\mH^\top)^{-1}$ and $\mP=\mH^\top(\mH\mH^\top)^{-1}\mH$ represent the pseudoinverse and the orthogonal projector operators for the degradation operator $\mH$. For $\mB_t=0$, the formulation in Eqn. \ref{eqn:c_pigfm} becomes equivalent to the $\Pi$GDM formulation proposed for OT-flows in \citet{pokle2024trainingfree}. For simplicity, since in this work, we only explore the parameterization in Eqn. \ref{eqn:c_pigfm} for $\mB_t=0$, we refer to this parameterization as \emph{$\Pi$GFM}.

 \begin{algorithm}[t]
\small
\caption{{\textit{Conjugate $\Pi$GFM sampling}}}
\begin{algorithmic}[1]

\State {\bfseries Input:} Corrupted observation $y$, Corruption operator $\mH$, Pretrained Flow $\vb_\vtheta(.,.)$, Choice of $\mB_t$, NFE budget $N$, Timestep discretization $\{t_i\}_{i=0}^N$, Flow kernel $\rvx_t = \alpha_t\rvx_1 + \gamma_t\rvz$, Start time $\tau$.
\State {\bfseries Output:} Clean sample $\hat{\rvx}_1$
\begin{tcolorbox}[algostyle3]
    \State Pre-Compute $\{\mA_{t_i}\}_{i=0}^N$ (Eqn. \ref{eq:at_flow}) \Comment{Pre-compute coefficients}
\State Pre-Compute $\{\mPhi_y^i, \mPhi_b^i, \mPhi_j^i\}_{i=0}^N$ (see Eqns. \ref{eq:flow_phiy}-\ref{eq:flow_phij})
\end{tcolorbox}
\begin{tcolorbox}[algostyle2]
\State $\rvz \sim \gN(0, \mI_d)$ \Comment{Draw initial samples from the generative prior}
\State $\rvx = \alpha_\tau \mH^\dag \vy + \gamma_\tau \rvz$ \Comment{Initialize using the pseudoinverse (See \citet{chung2022comecloserdiffusefaster})}
\State $\bar{\rvx} = \mA_\tau \rvx$ \Comment{Initial Projection Step}
\end{tcolorbox}
\begin{tcolorbox}[algostyle]
    \For{$n=0$ {\bfseries to} $N-1$}
\State $h = (t_{n+1} - t_n)$ \Comment{Time step differential}
\State $\rvx = \mA_{t_n}^{-1}\bar{\rvx}$
\State $\hat{\rvx}_1 = \frac{1}{\gamma_t\dot{\alpha}_t - \dot{\gamma}_t\alpha_t}\Big[-\dot{\gamma}_t\rvx_t + \gamma_t \vb_\theta(\rvx_t, t)\Big]$ \Comment{Tweedie's Estimate}
\State $\vv_l = h\mA_{t_n}\mB_{t_n}\mA_{t_n}^{-1}\bar{\rvx} + (\mPhi_y^{n+1} - \mPhi_y^{n})\vy$ \Comment{Linear drift}
\State $\vv_{nl} = (\mPhi_b^{n+1} - \mPhi_b^{n})\vb_\vtheta(\rvx, t_n) + (\mPhi_j^{n+1} - \mPhi_j^{n})\Big[\partial_{\rvx} \vb_\vtheta(\rvx, t_n) (\mH^\dag\vy - \mP\hat{\rvx}_1)\Big] $ \Comment{Non-Linear drift}
\State $\bar{\rvx} = \bar{\rvx} + \vv_l + \vv_{nl}$ \Comment{Euler Update}
\EndFor
\end{tcolorbox}
\noindent\Return $\rvx = \mA_{t_N}^{-1}\bar{\rvx}$ \Comment{Project back to original space when done}
\end{algorithmic}
\label{algo:cpigfm_algo}
\end{algorithm}

 \subsubsection{Conjugate-$\Pi$GFM (C-$\Pi$GFM)} Analogous to the discussion of C-$\Pi$GDM samplers in Section \ref{sec:cci_diffusion}. More specifically, given a noiseless linear inverse problem with $\sigma_y=0$, and the conditional score $\nabla_{\rvx_t} \log p(\rvy|\rvx_t)$, introducing the transformation $\bar{\rvx}_t=\mA_t\rvx_t$, where
    \begin{equation}
    \mA_t = \bm{\exp} \Big[\int_0^t \mB_s + \frac{w_s r_s^{-2}\gamma_t\dot{\gamma}_t^2}{2\alpha_t \Big(\gamma_t\dot{\alpha}_t - \dot{\gamma}_t\alpha_t\Big)}\mP ds\Big]
    \label{eq:at_flow}
    \end{equation}
    induces the following projected flow dynamics.
    \begin{equation}
        d\bar{\rvx}_t = \mA_t\mB_t\mA_t^{-1}\bar{\rvx}_t dt + d\mPhi_y \vy + d\mPhi_b \vb_\theta(\rvx_t, t) + d\mPhi_j \Big[\partial_{\rvx_t} \vb_\theta(\rvx_t, t) (\mH^\dag\vy - \mP\hat{\rvx}_1)\Big]
    \end{equation}
    where,
    \begin{equation}
        \mPhi_y = -\int_0^t \frac{w_s r_t^{-2}\gamma_s \dot{\gamma}_s}{\alpha_s}\mA_s\mH^\dag ds
        \label{eq:flow_phiy}
    \end{equation}
    \begin{equation}
        \mPhi_b = \int_0^t \mA_s \Big[\mI_d + \frac{w_s r_s^{-2}\gamma_s^2\dot{\gamma}_s}{\alpha_s (\gamma_s\dot{\alpha}_s - \dot{\gamma}_s\alpha_s)}\mP\Big] ds
        \label{eq:flow_phib}
    \end{equation}
    \begin{equation}
        \mPhi_j = \int_0^t \frac{w_s r_s^{-2} \gamma_s^2}{\alpha_s}\mA_s ds
        \label{eq:flow_phij}
    \end{equation}
    where $\bm{\exp(.)}$ denotes the matrix exponential, $\mH^\dag$, and $\mP$ are the pseudoinverse and projector operators (as defined previously). Lastly, the matrix $\mB_t$ is a design choice of our method. We specify a recipe for C-$\Pi$GFM sampling in Algorithm \ref{algo:cpigfm_algo}.

\section{Extension to Noisy and Non-linear Inverse Problems}
Here, we discuss an extension of Conditional Conjugate Integrators to noisy and non-linear inverse problems. While our discussion is primarily in the context of diffusion models, similar theoretical arguments also apply to Flows.
\label{sec:app_noisy_nl}

\textbf{Noisy Linear Inverse Problems:} For noisy linear inverse problems of the form,
\begin{equation}
    \rvy = \mH\rvx_0 + \sigma_y \rvz,
\end{equation}
for VPSDE diffusion, the \emph{noisy} transformation $\mA_t^{\sigma_y}$ can be approximated from the transformation $\mA_t$ for the noiseless case (i.e., $\sigma_y=0$) as illustrated in the following result (Proof in Appendix \ref{app:proof_5}):
    \begin{equation}
        \mA_t^{\sigma_y} = \mA_t + \kappa_3(t) \mH^\dag (\mH^\dag)^\top + \mathcal{O}(\sigma_y^4) \approx \mA_t + \kappa_3(t) \mH^\dag (\mH^\dag)^\top,
    \end{equation}
    \begin{equation}
        \kappa_3(t) = \frac{w\sigma_y^2}{2}\Big(\int_0^t \frac{\beta_s}{r_s^2}ds\Big) \Big[\exp \Big(\kappa_1(t) + \kappa_2(t)\Big) - 1\Big].
    \end{equation}
Consequently, the inverse projection $(\mA_t^{\sigma_y})^{-1}$ can be approximated from $\mA_t^{\sigma_y}$ from perturbation analysis. 
\begin{align}
    (\mA_t^{\sigma_y})^{-1} & \approx \mA_t^{-1} - \kappa_3(t)\mA_t^{-1}\mH^\dag(\mH^\dag)^\top\mA_t^{-1} + \mathcal{O}(\sigma_y^4)
\end{align}
Therefore, the transformation matrix $\mA_t^{\sigma_y}$ and its inverse (see Appendix \ref{app:proof_5}) can also be computed tractably for the noisy case. 
Since, for most practical purposes, $\sigma_y$ is pretty small, higher order terms in $\sigma_y^4$ can be safely ignored, making our approximation accurate. We include qualitative examples for 4x super-resolution with $\sigma_y=0.05$ for the ImageNet dataset in Figure \ref{fig:app_fig_4_noisy}

\textbf{Non-Linear Inverse Problems:} For non-linear inverse problems of the form,
\begin{equation}
    \vy = h(\rvx_0) + \sigma_y \rvz, \quad \rvz \sim \gN(0, \mI_d),
\end{equation}
similar to \citet{song2022pseudoinverse}, we heuristically re-define linear operations like $\mH^\dag \rvx_t$, $\mH\rvx_t$ and $\mP \rvx_t$ by their non-linear equivalents $h^\dag(\rvx_t)$, $h(\rvx_t)$ and $h^\dag(h(\rvx_t))$ respectively. Consequently, analogous to Eqn. \ref{eqn:at_simplified} the projection operator for a noiseless non-linear inverse problem, in this case, can be defined as,
\begin{equation}
    A_t = \exp(\kappa_1(t)) \Big[\mI_d + (\exp(\kappa_2(t)) - 1)P\Big], \quad \kappa_1(t) = \int_0^t \Big(\lambda + \frac{1}{2}\beta_s\Big) ds,\quad\kappa_2(t)=-\frac{w}{2}\int_0^t \beta_s ds,
\end{equation}
where $P = h^\dag(h(.))$ is non-linear `projector" operator.
For instance, in non-linear inverse problems like compression artifact removal, $h(\rvx_t)$ and $h^\dag(\rvx_t)$ can realized by encoders and decoders. We illustrate some qualitative examples in Figure~\ref{fig:app_fig_8}. It is worth noting that this is a purely heuristic approximation, and developing a more principled framework for non-linear inverse problems within our framework remains an interesting direction for further work.

\section{Implementation Details}
\label{app:in_the_wild}
In this section, we include additional practical implementation details for both C-$\Pi$GDM and C-$\Pi$GFM formulations.

\subsection{C-$\Pi$GDM: Practical Aspects}
\subsubsection{VP-SDE}
We work with the VP-SDE diffusion \citep{songscore} with the forward process specified as:
\begin{equation}
    d\rvx_t = -\frac{1}{2}\beta_t\rvx_t \, dt + \sqrt{\beta_t} \, d\rvw_t, \quad t \in [0, T],
\end{equation}
This implies, $\mF_t = -\frac{1}{2}\beta_t$ and $\mG_t = \sqrt{\beta_t}$. For the VP-SDE the perturbation kernel is given by, 
\begin{equation}
    p(\rvx_t|\rvx_0) = \mathcal{N}(\mu_t\rvx_0, \sigma_t^2 \mI_d)
\end{equation}
\begin{equation}
    \mu_t = \exp{\Big(-\frac{1}{2}\int_0^s \beta_s ds\Big)} \qquad \sigma_t^2 = \Big[1 - \exp{\Big(-\int_0^s \beta_s ds\Big)}\Big]
\end{equation}
The corresponding deterministic reverse process is parameterized as:
\begin{equation}
    d \rvx_t = -\frac{\beta_t}{2}\left[\rvx_t + \vs_\theta(\rvx_t,t)\right] \, dt.
\end{equation}
Moreover, we adopt the standard $\epsilon$-prediction parameterization which implies $\mC_\text{out}(t) = -1/\sigma_t$. Lastly, the Tweedies estimate $\hat{\rvx}_0$ can be specified as:
\begin{equation}
    \hat{\rvx}_0 = \frac{1}{\mu_t}\Big[\rvx_t + \sigma_t^2 \vs_\theta(\rvx_t, t)\Big]   
\end{equation}

\subsubsection{C-$\Pi$GDM - Simplified Expressions}
We choose the parameterization $\mB_t = \lambda \mI_d$ and set the adaptive guidance weight as $w_t=w\mu_t ^2 r_t^2$, where $r_t^2=\frac{\sigma_t^2}{\sigma_t^2 + \mu_t^2}$. The projected diffusion dynamics are then specified as:
\begin{equation}
    d\bar{\rvx}_t = \lambda\bar{\rvx}_t dt + d\mPhi_y \vy + d\mPhi_s \bm{\epsilon}_\theta(\rvx_t, t) + d\mPhi_j \Big[\partial_{\rvx_t} \bm{\epsilon_\theta(\rvx_t, t)} (\mH^\dag\vy - \mP\hat{\rvx}_0)\Big]
\end{equation}
where
\begin{equation}
    \mA_t = \bm{\exp} \Big[\int_0^t \Big(\lambda + \frac{1}{2}\beta_s\Big) ds \mI_d -\frac{w}{2}\Big(\int_0^t \beta_s ds\Big)\mP\Big]
\end{equation}
which further simplifies to,
\begin{equation}
    \mA_t = \exp(\kappa_1(t)) \Big[\mI_d + (\exp(\kappa_2(t)) - 1)\mP\Big], \quad \kappa_1(t) = \int_0^t \Big(\lambda + \frac{1}{2}\beta_s\Big) ds,\quad\kappa_2(t)=-\frac{w}{2}\int_0^t \beta_s ds
\end{equation}
Moreover, we have,
\begin{align}
    \mPhi_y &= -\int_0^t \frac{w_sr_s^{-2}}{2\mu_s}\mA_s\mG_s \mG_s^\top\mH^\dag ds \\
    &= -\int_0^t \frac{w\beta_t\mu_s}{2}\mA_s \mH^\dag ds \\
    &= -\int_0^t \frac{w\beta_t\mu_s}{2}\Big[\exp(\kappa_1(s)) \Big[\mI_d + (\exp(\kappa_2(s)) - 1)\mP\Big]\Big] \mH^\dag ds \\
    &= -\int_0^t \frac{w\beta_t\mu_s}{2}\exp(\kappa_1(s))\Big[\mH^\dag + (\exp(\kappa_2(s)) - 1)\mP\mH^\dag\Big] ds \\
    &= -\int_0^t \frac{w\beta_t\mu_s}{2}\exp(\kappa_1(s))\Big[\mH^\dag + (\exp(\kappa_2(s)) - 1)\mH^\dag\Big] ds \\
    &= -\Big[\int_0^t \frac{w\beta_t\mu_s}{2}\exp(\kappa_1(s) + \kappa_2(s))ds\Big]\mH^\dag
\end{align}
\begin{align}
    \mPhi_s &= -\int_0^t \frac{1}{2}\mA_s\mG_s \mG_s^\top\Big[\mI_d - \frac{w_sr_s^{-2}\sigma_s^2}{\mu_s^2}\mP \Big] \mC_\text{out}(s) ds \\
    &= \int_0^t \frac{\beta_s}{2\sigma_s}\mA_s\Big[\mI_d - w\sigma_s^2\mP \Big] ds \\
    &= \int_0^t \frac{\beta_s}{2\sigma_s}\mA_s ds - \Big[\int_0^t \frac{w\beta_s \sigma_s}{2}\exp(\kappa_1(s) + \kappa_2(s)) ds\Big] \mP \\
    &= \int_0^t \frac{\beta_s}{2\sigma_s}\exp(\kappa_1(s))\Big[\mI_d + (\exp(\kappa_2(s)) - 1)\mP\Big] ds - \Big[\int_0^t \frac{w\beta_s \sigma_s}{2}\exp(\kappa_1(s) + \kappa_2(s)) ds\Big] \mP \\
    &= \int_0^t \frac{\beta_s}{2\sigma_s}\exp(\kappa_1(s)) ds + \Big[\int_0^t \frac{\beta_s}{2\sigma_s}\exp(\kappa_1(s))(\exp(\kappa_2(s)) - 1)- \frac{w\beta_s \sigma_s}{2}\exp(\kappa_1(s) + \kappa_2(s)) ds\Big] \mP
\end{align}
\begin{align}
    \mPhi_j &= -\int_0^t \frac{w_sr_s^{-2}\sigma_s^2}{2\mu_s}\mA_s\mG_s \mG_s^\top\mC_\text{out}(s) ds = \int_0^t \frac{w\beta_s\mu_s\sigma_s}{2}\mA_s ds \\
    &= \int_0^t \frac{w\beta_s\mu_s\sigma_s}{2} \exp(\kappa_1(s)) \Big[\mI_d + (\exp(\kappa_2(s)) - 1)\mP\Big] ds \\
    &= \int_0^t \frac{w\beta_s\mu_s\sigma_s}{2} \exp(\kappa_1(s)) ds + \Big[\int_0^t \frac{w\beta_s\mu_s\sigma_s}{2} \exp(\kappa_1(s))(\exp(\kappa_2(s)) - 1) ds\Big] \mP 
\end{align}

\subsection{C-$\Pi$GFM: Practical Aspects}
\subsubsection{OT-Flows}
We work with OT-Flows \citep{albergo2023stochastic, lipman2023flow, liu2022flow} due to its wide adoption. More specifically, the corresponding interpolant can be specified as,
\begin{equation}
    \rvx_t = (1-t)\rvz + t\rvx_1, \quad \rvz \sim \mathcal{N}(0, \mI_d) \quad \rvx_1 \sim p_\text{data}
\end{equation}
For this case $\alpha_t=t$ and $\gamma_t = 1-t$. Therefore, the Tweedie's estimate of $\mathbb{E}(\rvx_t|\rvx_1)$ can be specified as (from Eqn. \ref{eqn:tweedie_flow}):
\begin{equation}
    \hat{\rvx}_1 = \rvx_t + (1 - t)\vb_\theta(\rvx_t, t)
\end{equation}

\subsubsection{C-$\Pi$GFM: Simplified Expressions}
We choose the parameterization $\mB_t = \lambda \mI_d$ and set the adaptive guidance weight as $w_t=w\alpha_t ^2 r_t^2$, where $r_t^2=\frac{\gamma_t^2}{\alpha_t^2 + \gamma_t^2}$. The projected diffusion dynamics are then specified as follows:
\begin{equation}
    d\bar{\rvx}_t = \lambda\bar{\rvx}_t dt + d\mPhi_y \vy + d\mPhi_b \vb_\theta(\rvx_t, t) + d\mPhi_j \Big[\partial_{\rvx_t} \vb_\theta(\rvx_t, t) (\mH^\dag\vy - \mP\hat{\rvx}_1)\Big]
\end{equation}
where,
 \begin{equation}
    \mA_t = \bm{\exp} \Big[\int_0^t \lambda \mI_d + \frac{wt(1-t)}{2}\mP ds\Big]
\end{equation}
which further simplifies to,
\begin{equation}
    \mA_t = \exp(\kappa_1(t)) \Big[\mI_d + (\exp(\kappa_2(t)) - 1)\mP\Big], \quad \kappa_1(t) = \int_0^t \lambda ds,\quad\kappa_2(t)=\frac{w}{2}\int_0^t s(1-s) ds
\end{equation}
Moreover, we have,
\begin{align}
    \mPhi_y &= -\int_0^t \frac{w_s r_t^{-2}\gamma_s \dot{\gamma}_s}{\alpha_s}\mA_s\mH^\dag ds\\
    &= -\int_0^t w\alpha_s\gamma_s \dot{\gamma}_s\mA_s\mH^\dag ds = \int_0^t ws(1-s)\mA_s\mH^\dag ds \\
    &= \int_0^t ws(1-s)\exp(\kappa_1(s))\Big[\mI_d + (\exp(\kappa_2(s)) - 1)\mP\Big]\mH^\dag ds \\
    &= \Big[\int_0^t ws(1-s)\exp(\kappa_1(s) + \kappa_2(s))ds\Big] \mH^\dag
\end{align}
\begin{align}
    \mPhi_b &= \int_0^t \mA_s \Big[\mI_d + \frac{w_s r_s^{-2}\gamma_s^2\dot{\gamma}_s}{\alpha_s (\gamma_s\dot{\alpha}_s - \dot{\gamma}_s\alpha_s)}\mP\Big] ds \\
    &= \int_0^t \mA_s \Big[\mI_d + \frac{w \alpha_s\gamma_s^2\dot{\gamma}_s}{(\gamma_s\dot{\alpha}_s - \dot{\gamma}_s\alpha_s)}\mP\Big] ds \\
    &= \int_0^t \mA_s \Big[\mI_d - w s(1-s)^2\mP\Big] ds \\
    &= \int_0^t \mA_s ds - \int_0^t ws(1-s)^2\mA_s\mP\Big] ds \\
    &= \int_0^t \mA_s ds - \Big[\int_0^t ws(1-s)^2\exp(\kappa_1(s) + \kappa_2(s))ds\Big] \mP \\
    &= \int_0^t \exp(\kappa_1(s)) ds + \Big[\int_0^t \exp(\kappa_1(s))(\exp(\kappa_2(s)) - 1) - ws(1-s)^2\exp(\kappa_1(s) + \kappa_2(s)) ds\Big] \mP
\end{align}
\begin{align}
    \mPhi_j &= \int_0^t \frac{w_s r_s^{-2} \gamma_s^2}{\alpha_s}\mA_s ds = \int_0^t w\alpha_s \gamma_s^2\mA_s ds \\
    &= \int_0^t w s (1-s)^2 \exp(\kappa_1(s))\Big[\mI_d + (\exp(\kappa_2(s)) - 1)\mP\Big] ds \\
    &= \int_0^t w s (1-s)^2 \exp(\kappa_1(s))ds + \Big[\int_0^t w s (1-s)^2\exp(\kappa_1(s))(\exp(\kappa_2(s)) - 1)ds\Big] \mP
\end{align}

\subsection{Coefficient Computation}
From the above analysis, most integrals are one-dimensional and can be computed in closed form or numerically with high precision. To clarify, with a predetermined timestep schedule \(\{t_i\}\), the coefficients \(\Phi\) can be calculated offline just once and then reused across various samples. Therefore, this computation must only be done once offline for each sampling run. For numerical approximation of these integrals, we use the \texttt{odeint} method from the \texttt{torchdiffeq} package \citep{torchdiffeq} with parameters \texttt{atol}=1e-5, \texttt{rtol}=1e-5 and the \texttt{RK45} solver \citep{DORMAND198019}. We set the initial value $\mPhi_\text{init} = \bm{0}$ for all coefficients $\mPhi$ as an initial condition for both C-$\Pi$GDM and C-$\Pi$GFM samplers.

\subsection{Choice of Numerical Solver} 
We use the Euler method to simulate projected diffusion/flow dynamics for simplicity. However, using higher-order numerical solvers within our framework is also possible. We leave this exploration to future work.

\subsection{Timestep Selection during Sampling}: 
We use uniform spacing for timestep discretization during sampling. We hypothesize our sampler can also benefit from more advanced timestep discretization techniques \citet{karraselucidating} commonly used for sampling in unconditional diffusion models in the low NFE regime.

\subsection{Last-Step Denoising}
It is common to add an Euler-based denoising step from a cutoff $\epsilon$ to zero to optimize for sample quality \citep{songscore, dockhornscore, jolicoeur-martineau2021adversarial} at the expense of another sampling step. In this work, we do not use last-step denoising for our samplers.

\subsection{Evaluation Metrics} 
We use the network function evaluations (NFE) to assess sampling efficiency and perceptual metrics KID \citep{binkowski2018demystifying}, LPIPS \citep{zhang2018unreasonable} and FID \citep{heusel2017gans} to assess sample quality. In practice, we use the \texttt{torch-fidelity}\citep{obukhov2020torchfidelity} package for computing all FID and KID scores reported in this work. For LPIPS, we use the \texttt{torchmetrics} package with \texttt{Alexnet} embedding.

\subsection{Baseline Hyperparameters} 
\paragraph{Diffusion Baselines:} For DPS \citep{chung2022diffusion}, we set NFE=1000 and set the step size for each task to the value recommended in Appendix D in \citet{chung2022diffusion}. For DDRM \citep{kawar2022denoising}, we set the number of sampling steps to NFE=20 with parameters $\eta_b=1.0$ and $\eta=0.85$ as recommended in \citet{kawar2022denoising}. For both DPS and DDRM we start diffusion sampling from $t=T$. For our implementation of $\Pi$-GDM, we set the start time parameter $\tau$ to 0.6 for super-resolution and deblurring. We set the guidance weight $w_t=wr_t^2$ where $w$ is tuned using grid search between 1.0 and 10.0 for best sample quality for super-resolution and deblurring. For implementation of all diffusion-based baselines, we use the official code for RED-Diff \citep{mardani2023variational} at \texttt{https://github.com/NVlabs/RED-diff}.

\paragraph{Flow Baselines:} In developing our flow-based baseline, we adhere to the approach outlined in $\Pi$GFM (Pokle et al., 2024), which advocates for a consistent guidance schedule characterized by \(w_t=w\) and \(r_t=\frac{\gamma_t^2}{\gamma_t^2 + \alpha_t^2}\). For each task, we perform a comprehensive grid search over the parameters \(\alpha_\tau=\{0.1, 0.2, \ldots, 0.7\}\) and \(w=\{1, 2, \ldots, 5\}\) (35 combinations in total) across different datasets to identify the optimal configuration that minimizes the LPIPS score. For the implementation of Flows, we use the official implementation of Rectified Flows \citep{liu2022flow} at \texttt{https://github.com/gnobitab/RectifiedFlow}.

\section{Additional Results}
\label{app:add_results}

\subsection{Additional Baseline Comparisons}
\begin{table}[t]
  \centering
  \footnotesize
  \begin{tabular}{@{}c|c|cc|cc|cc@{}}
\toprule
\multirow{2}{*}{\textbf{Task}}    & \multirow{2}{*}{\textbf{NFE}} & \multicolumn{2}{c|}{\textbf{LPIPS}$\downarrow$} & \multicolumn{2}{c|}{\textbf{KID$\times 10^{-3}$}$\downarrow$} & \multicolumn{2}{c}{\textbf{FID}$\downarrow$} \\ \cmidrule(l){3-8} 
                                  &                               & \textbf{C-$\Pi$GFM}              & \textbf{$\Pi$GFM}             & \textbf{C-$\Pi$GFM}                     & \textbf{$\Pi$GFM}                    & \textbf{C-$\Pi$GFM}            & \textbf{$\Pi$GFM}            \\ \midrule
\multirow{3}{*}{Inpainting}       & 5                             & \textbf{0.151}                   & 0.177                         & \textbf{6.5}                            & 15.6                                 & \textbf{21.76}                 & 30.82                        \\
                                  & 10                            & \textbf{0.122}                   & 0.136                         & \textbf{8.5}                            & 9.4                                  & \textbf{22.50}                 & 24.87                        \\
                                  & 20                            & \textbf{0.115}                   & 0.117                         & \textbf{6.4}                            & 10.4                                 & \textbf{20.39}                 & 24.42                        \\ \midrule
\multirow{3}{*}{Super-Resolution} & 5                             & \textbf{0.129}                   & 0.133                         & \textbf{4.1}                            & 5.7                                  & \textbf{18.43}                 & 19.55                        \\
                                  & 10                            & 0.132                            & \textbf{0.121}                & \textbf{4.0}                            & 4.6                                  & 18.32                          & \textbf{17.65}               \\
                                  & 20                            & 0.134                            & \textbf{0.119}                & 4.5                                     & \textbf{4.0}                         & 18.75                          & \textbf{16.97}               \\ \midrule
\multirow{3}{*}{Deblurring}       & 5                             & \textbf{0.176}                   & 0.177                         & \textbf{6.6}                            & 6.9                                  & \textbf{23.28}                 & 23.66                        \\
                                  & 10                            & 0.182                            & \textbf{0.164}                & 9.4                                     & \textbf{7.1}                         & 28.12                          & \textbf{23.62}               \\
                                  & 20                            & 0.191                            & \textbf{0.170}                & 12.4                                    & \textbf{7.2}                         & 31.76                          & \textbf{23.65}               \\ \bottomrule
\end{tabular}
  \caption{Quantitative evaluation on 4x superresolution, inpainting, and Gaussian deblurring on the AFHQ-Cat dataset.}
  \label{tab:afhq}
\end{table}

\begin{table}[ht]
  \centering
  \footnotesize
  \begin{tabular}{@{}c|c|cc|cc|cc@{}}
\toprule
\multirow{2}{*}{\textbf{Task}}    & \multirow{2}{*}{\textbf{NFE}} & \multicolumn{2}{c|}{\textbf{LPIPS}$\downarrow$} & \multicolumn{2}{c|}{\textbf{KID}$\times 10^{-3} \downarrow$} & \multicolumn{2}{c}{\textbf{FID}$\downarrow$} \\ \cmidrule(l){3-8} 
                                  &                               & \textbf{C-$\Pi$GFM}              & \textbf{$\Pi$GFM}             & \textbf{C-$\Pi$GFM}                    & \textbf{$\Pi$GFM}                    & \textbf{C-$\Pi$GFM}            & \textbf{$\Pi$GFM}            \\ \midrule
\multirow{3}{*}{Inpainting}       & 5                             & \textbf{0.208}                   & -                             & \textbf{7.0}                           & -                                    & \textbf{45.66}                 & -                            \\
                                  & 10                            & \textbf{0.176}                   & -                             & \textbf{4.4}                           & -                                    & \textbf{40.69}                 & -                            \\
                                  & 20                            & \textbf{0.167}                   & -                             & \textbf{4.2}                           & -                                    & \textbf{40.35}                 & -                            \\ \midrule
\multirow{3}{*}{Super-Resolution} & 5                             & \textbf{0.174}                   & 0.219                         & \textbf{3.1}                           & 7.7                                  & \textbf{37.54}                 & 46.03                        \\
                                  & 10                            & \textbf{0.150}                   & 0.193                         & \textbf{1.1}                           & 4.6                                  & \textbf{32.41}                 & 37.34                        \\
                                  & 20                            & \textbf{0.148}                   & 0.175                         & \textbf{0.9}                           & 2.5                                  & 32.26                          & \textbf{32.15}               \\ \midrule
\multirow{3}{*}{Deblurring}       & 5                             & \textbf{0.209}                   & 0.220                         & \textbf{5.0}                           & 9.0                                  & \textbf{44.78}                 & 49.27                        \\
                                  & 10                            & 0.204                            & \textbf{0.193}                & 10.7                                   & \textbf{4.7}                         & 53.53                          & \textbf{44.21}               \\
                                  & 20                            & 0.224                            & \textbf{0.175}                & 18.0                                   & \textbf{3.5}                         & 62.87                          & \textbf{39.95}               \\ \bottomrule
\end{tabular}
  \caption{Quantitative evaluation on 4x superresolution, inpainting, and Gaussian deblurring on the LSUN-Bedroom dataset. We note that $\Pi$GFM fails to generate reasonable texture in the masked region even with the maximum NFE=20, so we choose not to report the results here. (See qualitative examples in Figure~\ref{fig:app_fig_5})}
  \label{tab:lsun}
\end{table}

\begin{table}[t]
  \centering
  \scriptsize
  \setlength{\tabcolsep}{2pt}
\begin{tabular}{@{}c|c|cccc|cccc|cccc@{}}
\toprule
\multirow{2}{*}{\textbf{Task}}    & \multirow{2}{*}{\textbf{NFE}} & \multicolumn{4}{c|}{\textbf{LPIPS}$\downarrow$}                          & \multicolumn{4}{c|}{\textbf{KID}$\times 10^{-3} \downarrow$}            & \multicolumn{4}{c}{\textbf{FID}$\downarrow$}                              \\ \cmidrule(l){3-14} 
                                  &                               & \textbf{C-$\Pi$GDM} & \textbf{$\Pi$GDM} & \textbf{DPS}           & \textbf{DDRM}          & \textbf{C-$\Pi$GDM} & \textbf{$\Pi$GDM} & \textbf{DPS}           & \textbf{DDRM}         & \textbf{C-$\Pi$GDM} & \textbf{$\Pi$GDM} & \textbf{DPS}            & \textbf{DDRM}          \\ \midrule
\multirow{3}{*}{Super-Resolution} & 5                             & \textbf{0.095}      & 0.133             & \multirow{3}{*}{0.106} & \multirow{3}{*}{0.106} & \textbf{10.9}       & 17.4              & \multirow{3}{*}{7.8}   & \multirow{3}{*}{22.8} & \textbf{32.01}      & 41.39             & \multirow{3}{*}{30.86}  & \multirow{3}{*}{36.95} \\
                                  & 10                            & \textbf{0.086}      & 0.106             &                        &                        & \textbf{8.8}        & 10.2              &                        &                       & \textbf{29.07}      & 32.79             &                         &                        \\
                                  & 20                            & \textbf{0.083}      & 0.087             &                        &                        & 5.8                 & \textbf{4.6}      &                        &                       & \textbf{26.37}      & 26.17             &                         &                        \\ \midrule
\multirow{3}{*}{Deblurring}       & 5                             & \textbf{0.127}      & 0.147             & \multirow{3}{*}{0.348} & \multirow{3}{*}{0.132} & \textbf{7.3}        & 14.6              & \multirow{3}{*}{109.4} & \multirow{3}{*}{11.5} & \textbf{31.18}      & 39.63             & \multirow{3}{*}{142.26} & \multirow{3}{*}{33.94} \\
                                  & 10                            & \textbf{0.111}      & 0.123             &                        &                        & \textbf{6.3}        & 7.7               &                        &                       & \textbf{29.08}      & 31.49             &                         &                        \\
                                  & 20                            & 0.112               & \textbf{0.103}    &                        &                        & 4.4                 & \textbf{3.1}      &                        &                       & 27.68               & \textbf{26.30}    &                         &                        \\ \bottomrule
\end{tabular}
  \caption{Quantitative evaluation on 4x superresolution and Gaussian Deblurring tasks for the FFHQ dataset. DPS was evaluated with NFE=1000 but failed to perform well on the deblurring task. DDRM was evaluated with NFE=20.}
  \label{tab:ffhq}
\end{table}

We include additional comparisons between our proposed samplers and competing baselines on the AFHQ-Cat (see Table \ref{tab:afhq}), LSUN Bedroom (see Table \ref{tab:lsun}), and the FFHQ (see Table \ref{tab:ffhq}) datasets.

\textbf{A note on Inpainting evaluations for ImageNet.} We find that for diffusion model evaluations, the continuous sampler for $\Pi$GDM suffers from noisy artifacts for the inpainting task. Consequently, Conjugate $\Pi$GDM suffers from similar artifacts. Therefore, we do not report results on this task for the ImageNet dataset.

\subsection{Comparison of Perceptual vs Recovery Metrics}
Here, we highlight the robustness of C-$\Pi$GDM in both perceptual and recovery metrics in the context of inverse problems. For completeness, we provide a comparison between DPS, $\Pi$GDM, and C-$\Pi$GDM in terms of PSNR, SSIM, FID, and LPIPS in Tables \ref{table:real_perc_1} and \ref{table:real_perc_2} on the ImageNet-256 and FFHQ-256 datasets on the 4x super-resolution task. It is worth noting that the PSNR and SSIM scores for all methods correspond with the best FID/LPIPS scores presented in the main text for these methods. Our method achieves competitive PSNR and SSIM scores for better perceptual quality than competing baselines like DPS/$\Pi$-GDM, even for very small sampling budgets. For instance, on the FFHQ dataset, our method achieves a PSNR of 28.97 compared to 28.49 for DPS while achieving better perceptual sample quality (LPIPS: 0.095 for ours vs 0.107 for DPS) and requiring around 200 times less sampling budget (NFE=5 for our method vs 1000 for DPS). Therefore, we argue that our perceptual quality to recovery trade-off is better than competing baselines.

\begin{table}[t]
\centering
\footnotesize
\begin{tabular}{@{}ccccc@{}}
\toprule
\multicolumn{1}{l}{\textbf{}}              & \multicolumn{1}{l}{\textbf{PSNR} $\uparrow$}     & \multicolumn{1}{l}{\textbf{SSIM} $\uparrow$}     & \multicolumn{1}{l}{\textbf{FID} $\downarrow$}      & \multicolumn{1}{l}{\textbf{LPIPS} $\downarrow$}    \\ \midrule
{\color[HTML]{333333} DPS (NFE=1000)}      & {\color[HTML]{333333} \textbf{23.81}} & {\color[HTML]{333333} \textbf{0.708}} & {\color[HTML]{333333} 38.18}          & {\color[HTML]{333333} 0.252}          \\
{\color[HTML]{333333} $\Pi$GDM (NFE=20)}   & {\color[HTML]{333333} 21.92}          & {\color[HTML]{333333} 0.646}          & {\color[HTML]{333333} 37.36}          & {\color[HTML]{333333} 0.222}          \\
{\color[HTML]{333333} C-$\Pi$GDM (NFE=5)}  & {\color[HTML]{333333} 22.32}          & {\color[HTML]{333333} 0.641}          & {\color[HTML]{333333} 37.31}          & {\color[HTML]{333333} 0.220}          \\
{\color[HTML]{333333} C-$\Pi$GDM (NFE=10)} & {\color[HTML]{333333} 23.00}          & {\color[HTML]{333333} 0.651}          & {\color[HTML]{333333} \textbf{34.22}} & {\color[HTML]{333333} \textbf{0.206}} \\
{\color[HTML]{333333} C-$\Pi$GDM (NFE=20)} & {\color[HTML]{333333} 23.16}          & {\color[HTML]{333333} 0.654}          & {\color[HTML]{333333} 34.28}          & {\color[HTML]{333333} 0.207} \\ \bottomrule         
\end{tabular}
\caption{Comparison between C-$\Pi$GDM and other baselines in terms of the Recovery (a.k.a distortion) vs Perception tradeoff for ImageNet-256 dataset for the SR(x4) task.}
\label{table:real_perc_1}
\end{table}

\begin{table}[t]
\centering
\begin{tabular}{@{}ccccc@{}}
\toprule
\textbf{}                                  & \textbf{PSNR} $\uparrow$                         & \textbf{SSIM} $\uparrow$                         & \textbf{FID} $\downarrow$                          & \textbf{LPIPS} $\downarrow$                        \\ \midrule
\color[HTML]{333333} DPS (NFE=1000)      & {\color[HTML]{333333} \textbf{28.49}} & {\color[HTML]{333333} \textbf{0.834}} & {\color[HTML]{333333} 30.86}          & {\color[HTML]{333333} 0.107}          \\
\color[HTML]{333333} $\Pi$GDM (NFE=20)   & {\color[HTML]{333333} 28.26}          & {\color[HTML]{333333} 0.818}          & {\color[HTML]{333333} \textbf{26.17}} & {\color[HTML]{333333} 0.087}          \\
\color[HTML]{333333} C-$\Pi$GDM (NFE=5)  & {\color[HTML]{333333} 28.97}          & {\color[HTML]{333333} 0.832}          & {\color[HTML]{333333} 32.01}          & {\color[HTML]{333333} 0.095}          \\
\color[HTML]{333333} C-$\Pi$GDM (NFE=10) & {\color[HTML]{333333} 29.03}          & {\color[HTML]{333333} 0.821}          & {\color[HTML]{333333} 29.07}          & {\color[HTML]{333333} 0.086}          \\
\color[HTML]{333333} C-$\Pi$GDM (NFE=20) & {\color[HTML]{333333} 28.79}          & {\color[HTML]{333333} 0.809}          & {\color[HTML]{333333} 26.37}          & {\color[HTML]{333333} \textbf{0.083}} \\ \bottomrule
\end{tabular}
\caption{Comparison between C-$\Pi$GDM and other baselines in terms of the Recovery (a.k.a distortion) vs Perception tradeoff for FFHQ-256 dataset for the SR(x4) task.}
\label{table:real_perc_2}
\end{table}

\subsection{Traversing the Recovery vs Perceptual trade-off}
In addition to the guidance weight $w$, our method also allows tuning an additional hyperparameter $\lambda$, which controls the dynamics of the projection operator (See Sections 2.3 and 3.2 for more intuition). Therefore, tuning $w$ and $\lambda$ can help traverse the trade-off curve between perceptual quality and distortion for a fixed NFE budget. We illustrate this aspect in Table \ref{table:impact_w} (fixed $\lambda$ with varying $w$) and Table \ref{table:impact_lam} (fixed $w$ with varying $\lambda$) for the SR(x4) task on the ImageNet-256 dataset using the PSNR, LPIPS, and FID metrics. Therefore, our method offers greater flexibility to tune the sampling process towards either good perceptual quality or good recovery for a given application while maintaining the same number of sampling steps. In contrast, other methods like DPS or $\Pi$-GDM do not offer such flexibility. Moreover, tuning the guidance weight in methods like DPS could be very expensive due to its high sampling budget requirement (around 1000 NFE).

\subsection{Qualitative Results}
\label{app:qual}
\textbf{Diffusion Models:} 
\begin{enumerate}
    \item We include additional qualitative comparisons between $\Pi$-GDM and our proposed C-$\Pi$GDM sampler for the ImageNet dataset in Fig. \ref{fig:app_fig_1}.
    \item We include a qualitative comparison between sample quality at different sampling budgets for the C-$\Pi$GDM sampler in Fig.\ref{fig:app_fig_2}.
    \item We qualitatively study the impact of varying $w$ on sample quality in Fig. \ref{fig:app_fig_3} and the impact of varying $\lambda$ on sample quality in Fig. \ref{fig:app_fig_4}.
    \item We qualitatively present the performance of the C-$\Pi$GDM sampler for noisy inverse problems in Fig. \ref{fig:app_fig_4_noisy}. In just 5 steps, our method can also generate good-quality samples for noisy inverse problems.
\end{enumerate}

\begin{table}[t]
\centering
\begin{tabular}{@{}cccc@{}}
\toprule
\textbf{w}                & \textbf{PSNR} $\uparrow$                         & \textbf{LPIPS} $\downarrow$                        & \textbf{FID} $\downarrow$                          \\ \midrule
{\color[HTML]{333333} 2}  & {\color[HTML]{333333} 22.91}          & {\color[HTML]{333333} 0.339}          & {\color[HTML]{333333} 48.48}          \\
{\color[HTML]{333333} 4}  & {\color[HTML]{333333} 23.37}          & {\color[HTML]{333333} 0.306}          & {\color[HTML]{333333} 45.03}          \\
{\color[HTML]{333333} 6}  & {\color[HTML]{333333} \textbf{23.49}} & {\color[HTML]{333333} 0.274}          & {\color[HTML]{333333} 42.68}          \\
{\color[HTML]{333333} 8}  & {\color[HTML]{333333} 23.44}          & {\color[HTML]{333333} 0.266}          & {\color[HTML]{333333} 40.96}          \\
{\color[HTML]{333333} 10} & {\color[HTML]{333333} 23.28}          & {\color[HTML]{333333} 0.254}          & {\color[HTML]{333333} 40.27}          \\
{\color[HTML]{333333} 12} & {\color[HTML]{333333} 22.89}          & {\color[HTML]{333333} 0.246}          & {\color[HTML]{333333} \textbf{40.13}} \\
{\color[HTML]{333333} 14} & {\color[HTML]{333333} 22.74}          & {\color[HTML]{333333} \textbf{0.239}} & {\color[HTML]{333333} 40.16}          \\ \bottomrule
\end{tabular}
\caption{Illustration of the impact of $w$
 for a fixed $\lambda=0.0$ on the sample recovery (PSNR) vs sample perceptual quality (LPIPS, FID) at NFE=5 for our method. The task is SR(x4) on the ImageNet-256 dataset.}
\label{table:impact_w}
\end{table}

\begin{table}[t]
\centering
\begin{tabular}{@{}cccc@{}}
\toprule
\textbf{$\lambda$} & \textbf{PSNR} $\uparrow$ & \textbf{LPIPS} $\downarrow$ & \textbf{FID} $\downarrow$ \\ \midrule
-1.0               & 20.96         & 0.291          & 42.56        \\
-0.8               & 21.33         & 0.265          & 40.97        \\
-0.6               & 21.69         & 0.240          & 39.38        \\
-0.4               & 22.04         & 0.223          & 37.83        \\
-0.2               & 22.32         & \textbf{0.220}          & \textbf{37.31}        \\
0.2                & 22.73         & 0.257          & 45.27        \\
0.4                & 22.90         & 0.275          & 48.98        \\
0.6                & 23.03         & 0.283          & 47.47        \\
0.8                & 23.11         & 0.285          & 46.2         \\
1.0                & \textbf{23.15}         & 0.285          & 46.41        \\ \bottomrule
\end{tabular}
\caption{Illustration of the impact of $\lambda$
 for a fixed $w=15.0$ on the sample recovery (PSNR) vs sample perceptual quality (LPIPS, FID) at NFE=5 for our method. The task is SR(x4) on the ImageNet-256 dataset.}
\label{table:impact_lam}
\end{table}

\textbf{Flow Models:} 
\begin{enumerate}
    \item We include additional qualitative comparisons between $\Pi$-GFM and our proposed C-$\Pi$GFM sampler with different sampling budget for the all three datasets in Fig. \ref{fig:app_fig_5}.
    \item We qualitatively study the impact of varying $w$ on sample quality in Fig. \ref{fig:app_fig_6} and the impact of varying $\lambda$ on sample quality in Fig. \ref{fig:app_fig_7}.
\end{enumerate}

\begin{figure}[ht]
    \centering
    \includegraphics[width=1.0\textwidth]{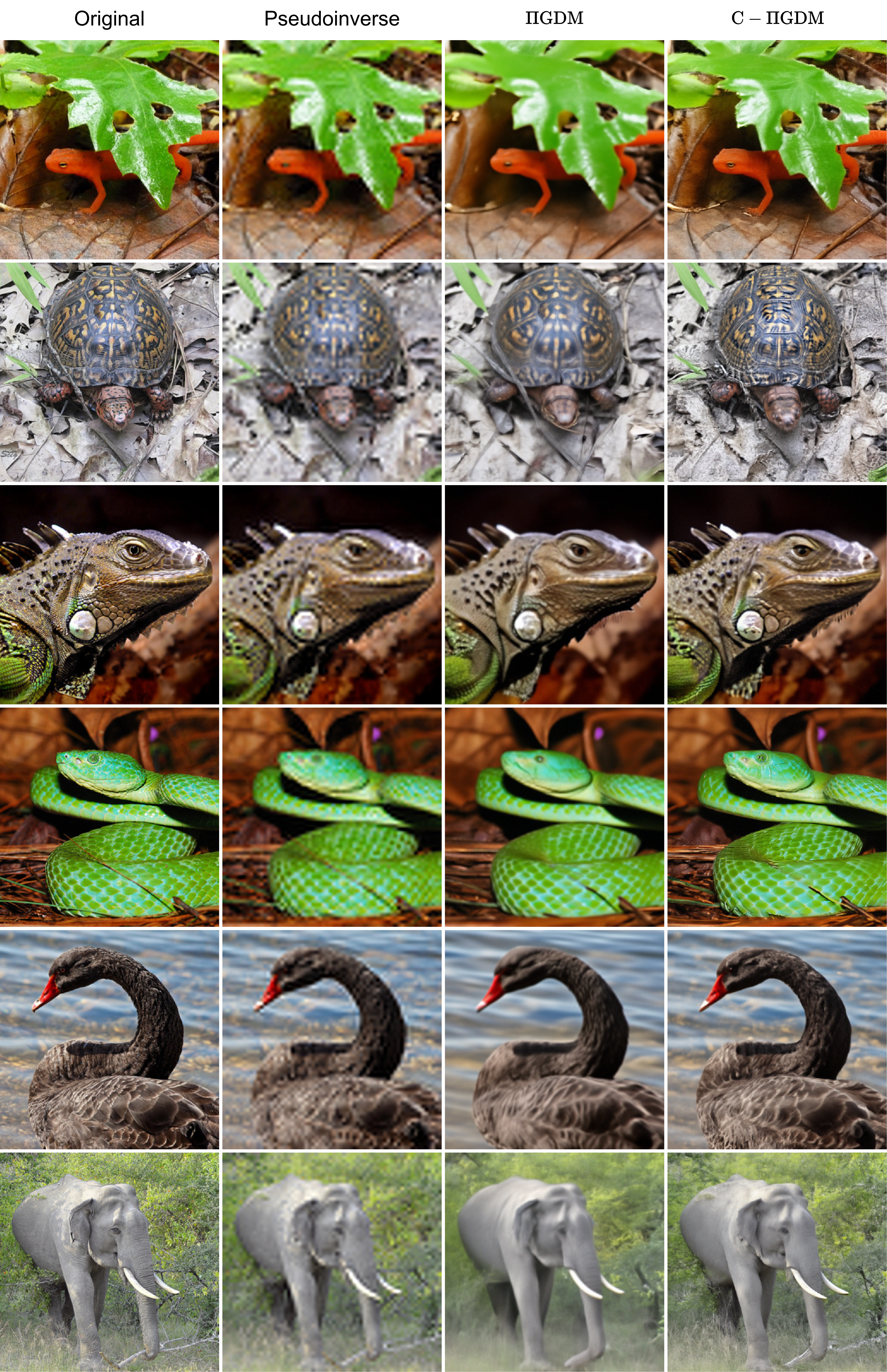}
    \caption{Qualitative comparison between $\Pi$GDM and C-$\Pi$GDM at NFE=5 for the ImageNet dataset on the 4x Superresolution task. C-$\Pi$GDM can generate high-frequency details even for a low compute budget as compared to the baseline $\Pi$-GDM (Best Viewed when zoomed in)}
    \label{fig:app_fig_1}
    \vspace{-1em}
\end{figure}

\begin{figure}[ht]
    \centering
    \includegraphics[width=1.0\textwidth]{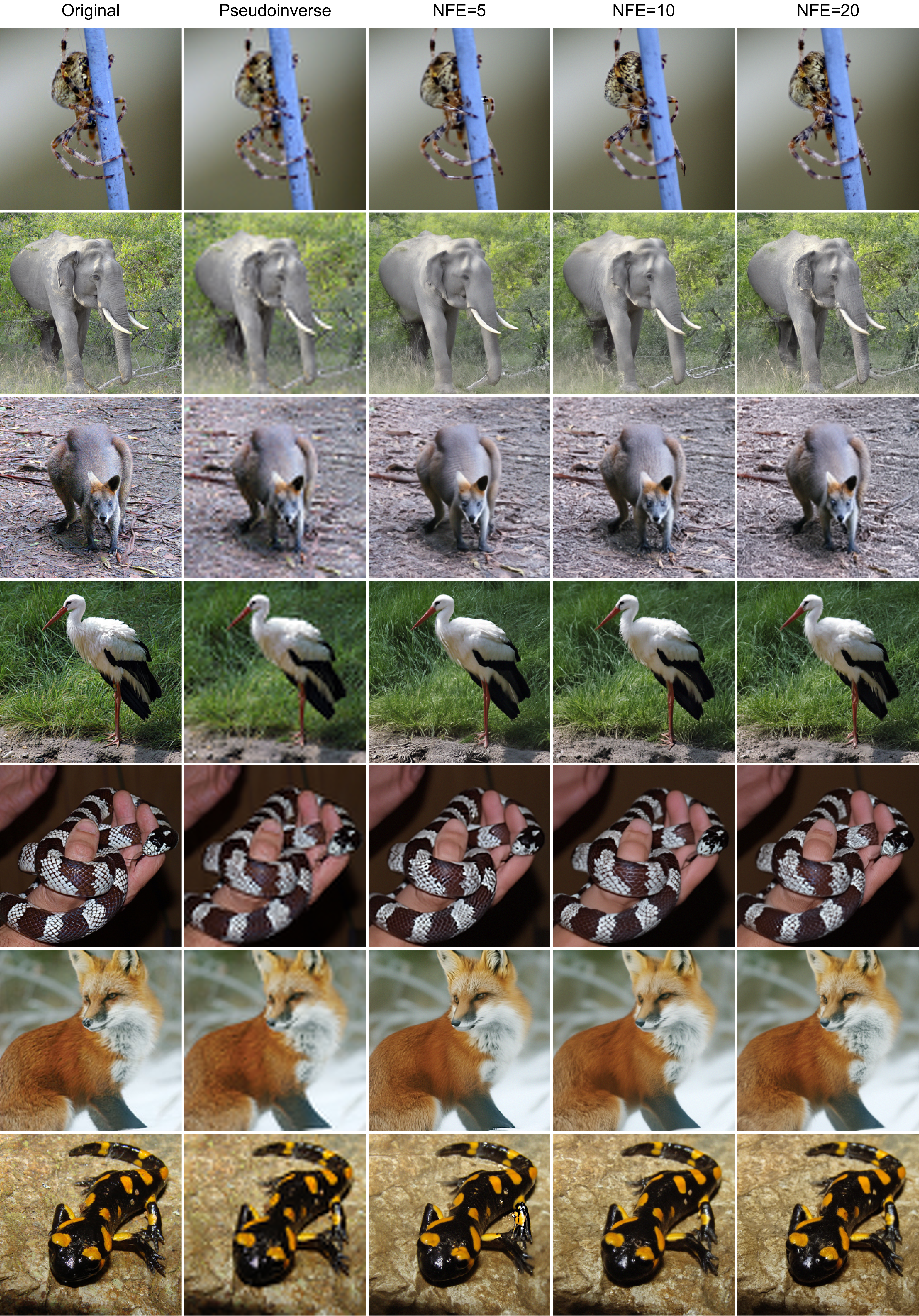}
    \caption{Qualitative comparison for different sampling budgets for the ImageNet dataset on the 4x Superresolution task. C-$\Pi$GDM can generate high-quality samples in just 5 steps (Best Viewed when zoomed in)}
    \label{fig:app_fig_2}
\end{figure}

\begin{figure}[ht]
    \centering
    \includegraphics[width=1.0\textwidth]{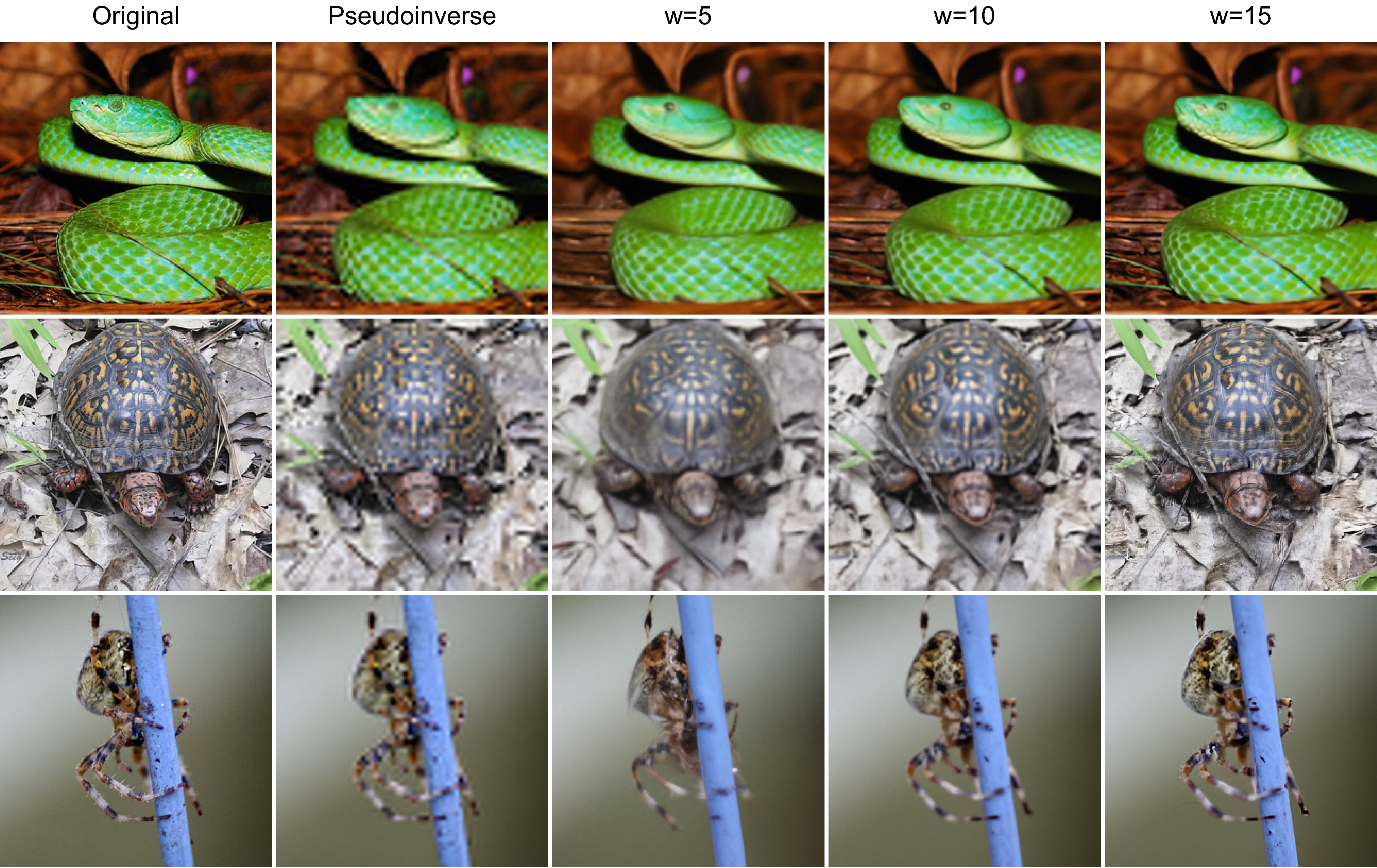}
    \caption{Impact of varying C-$\Pi$GDM guidance weight $w$ on sample quality for the ImageNet dataset on the 4x Superresolution task. High guidance weight is crucial to generate good quality samples from C-$\Pi$GDM (NFE=5 steps) (Best Viewed when zoomed in)}
    \label{fig:app_fig_3}
\end{figure}

\begin{figure}[ht]
    \centering
    \includegraphics[width=1.0\textwidth]{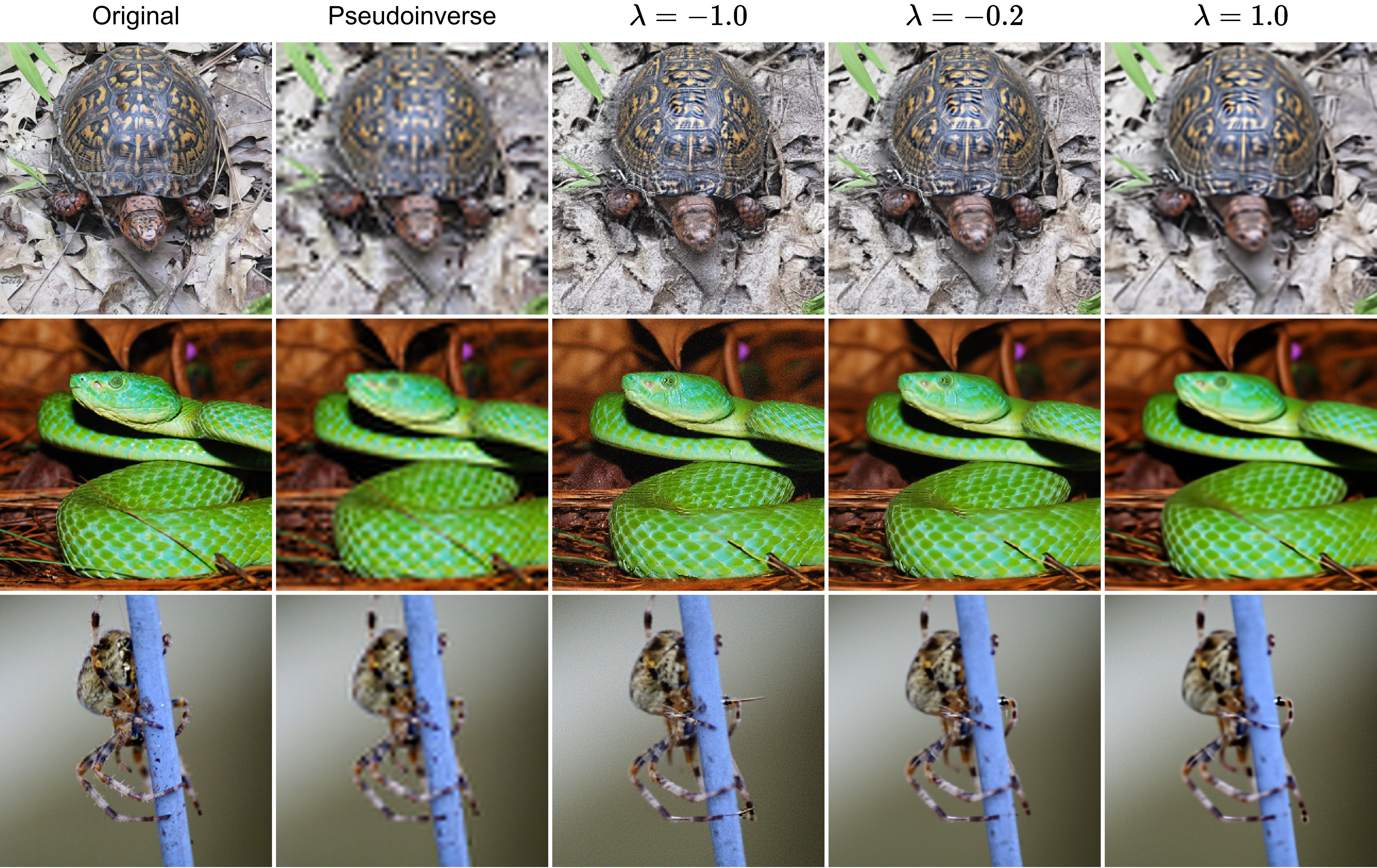}
    \caption{Impact of varying C-$\Pi$GDM $\lambda$ on sample quality for the ImageNet dataset on the 4x Superresolution task. High $\lambda$ can lead to blurry samples while a very low $\lambda$ can lead to over-sharpened artifacts (NFE=5 steps) (Best Viewed when zoomed in)}
    \label{fig:app_fig_4}
\end{figure}

\begin{figure}[ht]
    \centering
    \includegraphics[width=0.9\textwidth]{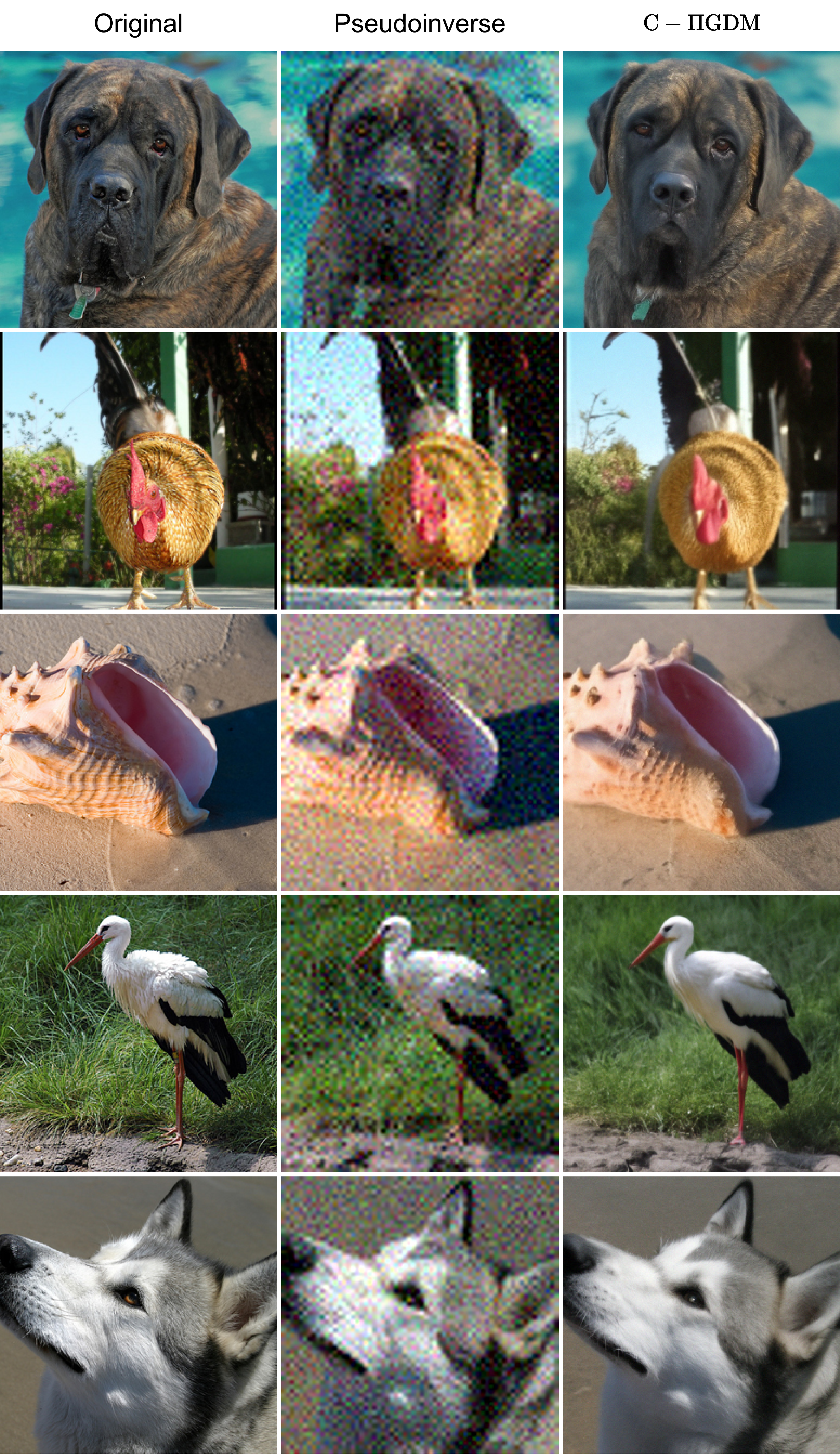}
    \caption{C-$\Pi$GDM can also generate good quality samples for noisy inverse problems (4x superres with NFE=5, $\sigma_y=0.05$). For this case naively computing the pseudoinverse fails to get rid of the noise.}
    \label{fig:app_fig_4_noisy}
    \vspace{-1em}
\end{figure}

\begin{figure}[ht]
    \centering
    \includegraphics[width=1.0\textwidth]{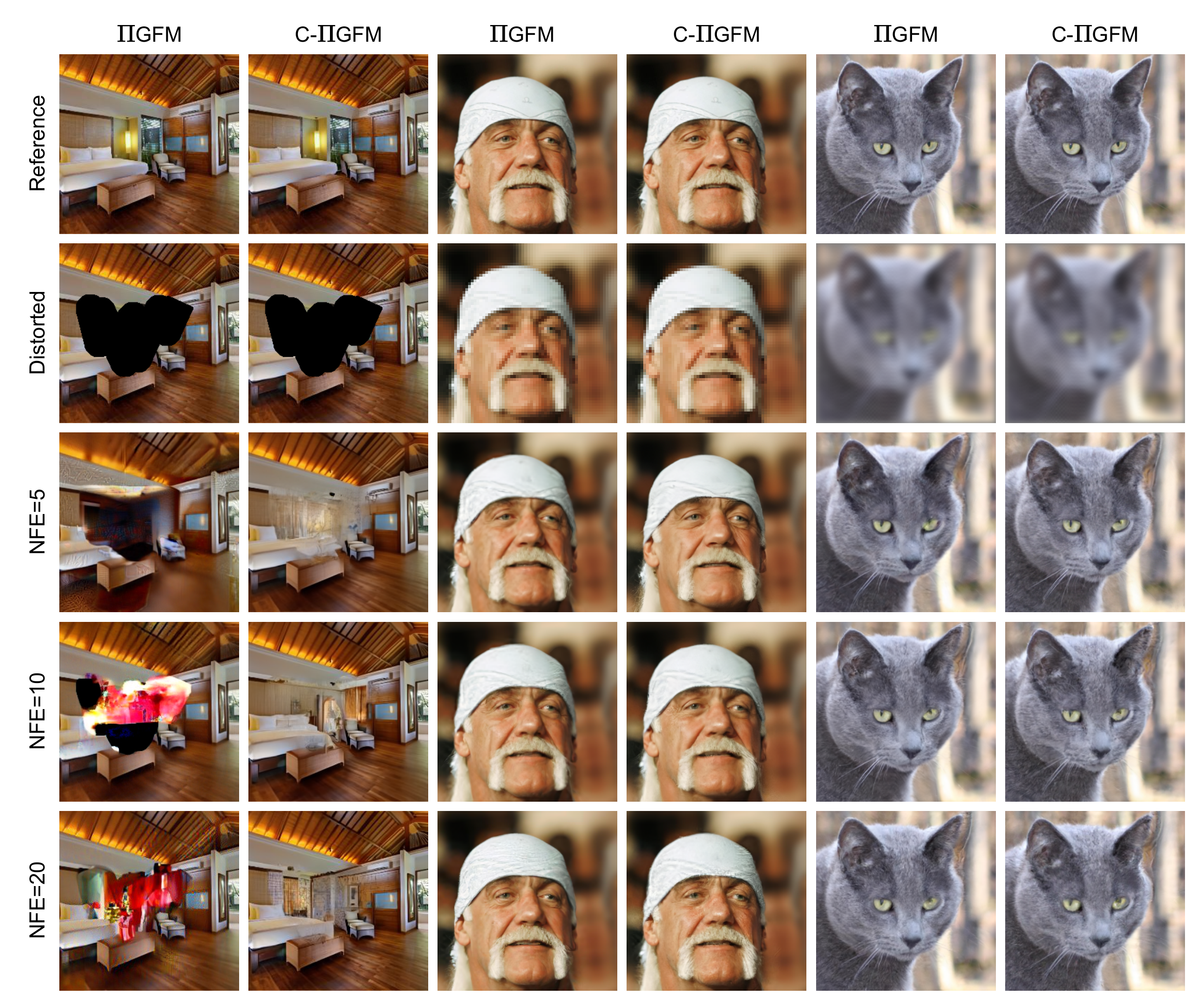}
    \caption{Qualitative comparison between $\Pi$GFM and C-$\Pi$GFM at NFE=\{5, 10\} for the 3 datasets on 3 tasks. C-$\Pi$GFM can generate high-frequency details even for a low compute budget as compared to the baseline $\Pi$-GFM (Best Viewed when zoomed in). We did not report $\Pi$GFM inpainting results in Table~\ref{tab:lsun} as it failed to generate ``reasonable'' textures even after extensive hyper-parameter search on $w$ and $\tau$.}
    \label{fig:app_fig_5}
\end{figure}

\begin{figure}[ht]
    \centering
    \includegraphics[width=1.0\textwidth]{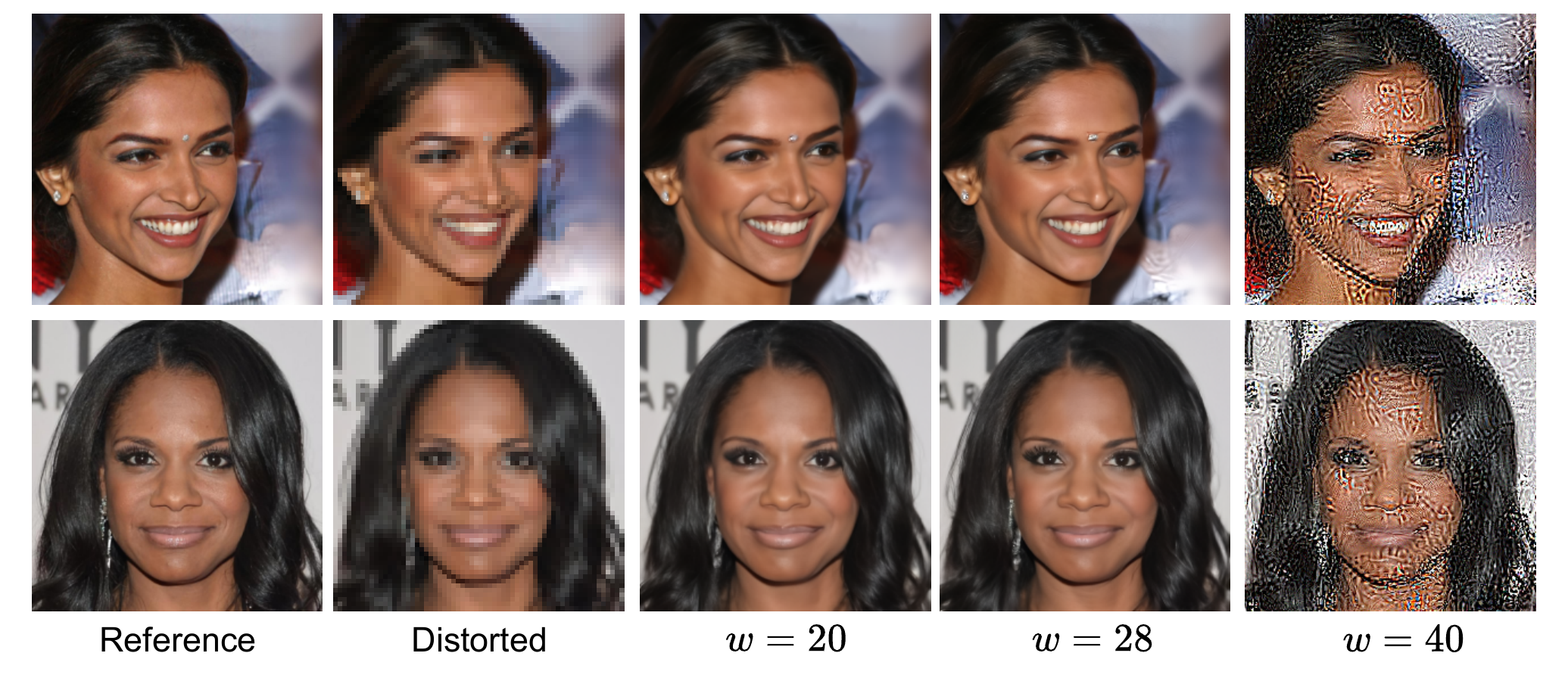}
    \caption{Impact of varying C-$\Pi$GFM guidance weight $w$ on sample quality for the ImageNet dataset on the 4x Superresolution task. High guidance weight is crucial to generate good quality samples from C-$\Pi$GFM (NFE=5 steps) (Best Viewed when zoomed in)}
    \label{fig:app_fig_6}
\end{figure}

\begin{figure}[ht]
    \centering
    \includegraphics[width=1.0\textwidth]{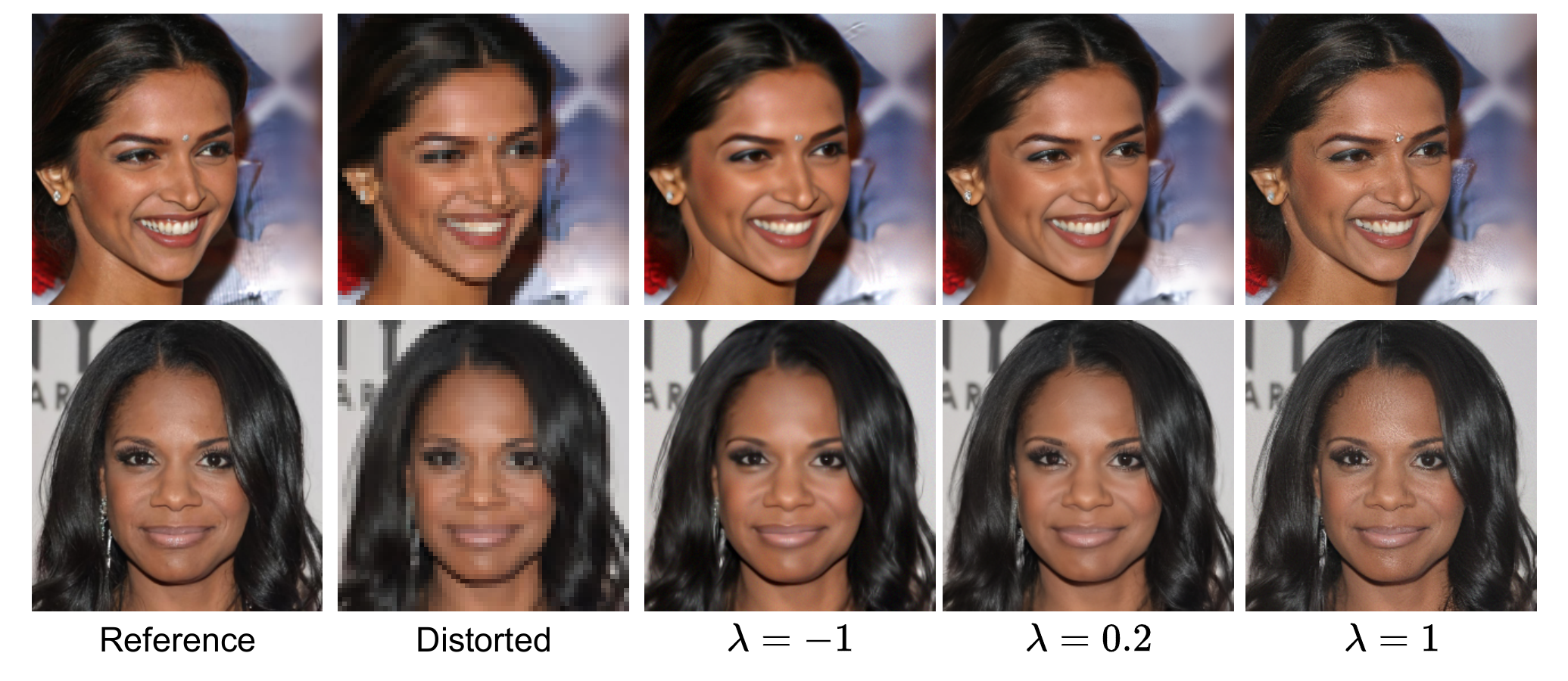}
    \caption{Impact of varying C-$\Pi$GFM $\lambda$ on sample quality for the ImageNet dataset on the 4x Superresolution task. High $\lambda$ can lead to blurry samples while a very high $\lambda$ can lead to over-sharpened artifacts (NFE=5 steps) (Best Viewed when zoomed in)}
    \label{fig:app_fig_7}
\end{figure}

\begin{figure}[ht]
    \centering
    \includegraphics[width=1.0\textwidth]{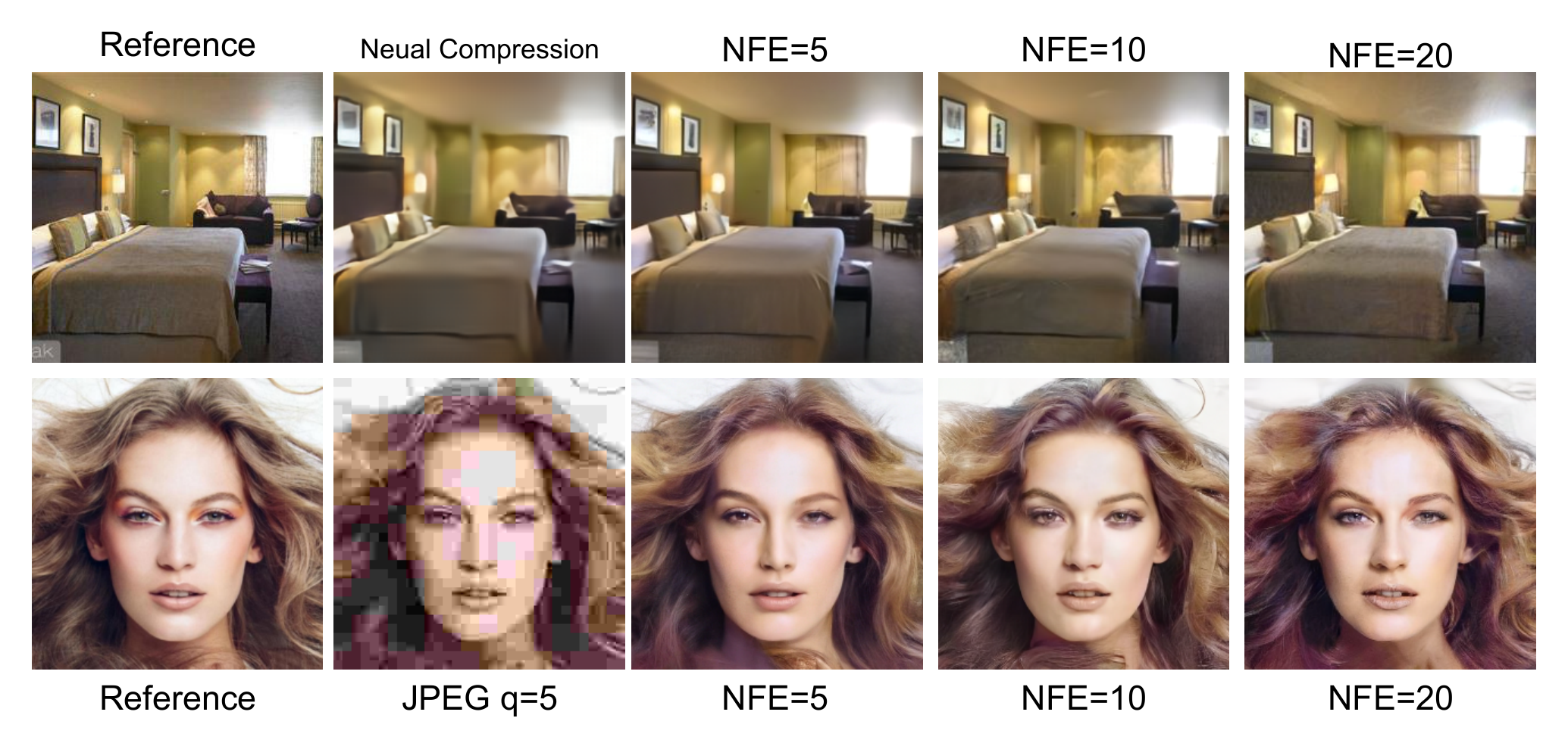}
    \caption{C-$\Pi$GFM for solving compression inverse problem. Top: decoding compressed latents from pretrained mean-scale hyperprior neural codec~\citep{NEURIPS2018_53edebc5}; Bottom: JPEG image restoration.}
    \label{fig:app_fig_8}
\end{figure}

\clearpage
\newpage
\section*{NeurIPS Paper Checklist}

The checklist is designed to encourage best practices for responsible machine learning research, addressing issues of reproducibility, transparency, research ethics, and societal impact. Do not remove the checklist: {\bf The papers not including the checklist will be desk rejected.} The checklist should follow the references and precede the (optional) supplemental material.  The checklist does NOT count towards the page
limit. 

Please read the checklist guidelines carefully for information on how to answer these questions. For each question in the checklist:
\begin{itemize}
    \item You should answer \answerYes{}, \answerNo{}, or \answerNA{}.
    \item \answerNA{} means either that the question is Not Applicable for that particular paper or the relevant information is Not Available.
    \item Please provide a short (1–2 sentence) justification right after your answer (even for NA). 
\end{itemize}

{\bf The checklist answers are an integral part of your paper submission.} They are visible to the reviewers, area chairs, senior area chairs, and ethics reviewers. You will be asked to also include it (after eventual revisions) with the final version of your paper, and its final version will be published with the paper.

The reviewers of your paper will be asked to use the checklist as one of the factors in their evaluation. While "\answerYes{}" is generally preferable to "\answerNo{}", it is perfectly acceptable to answer "\answerNo{}" provided a proper justification is given (e.g., "error bars are not reported because it would be too computationally expensive" or "we were unable to find the license for the dataset we used"). In general, answering "\answerNo{}" or "\answerNA{}" is not grounds for rejection. While the questions are phrased in a binary way, we acknowledge that the true answer is often more nuanced, so please just use your best judgment and write a justification to elaborate. All supporting evidence can appear either in the main paper or the supplemental material, provided in appendix. If you answer \answerYes{} to a question, in the justification please point to the section(s) where related material for the question can be found.

IMPORTANT, please:
\begin{itemize}
    \item {\bf Delete this instruction block, but keep the section heading ``NeurIPS paper checklist"},
    \item  {\bf Keep the checklist subsection headings, questions/answers and guidelines below.}
    \item {\bf Do not modify the questions and only use the provided macros for your answers}.
\end{itemize}

\begin{enumerate}

\item {\bf Claims}
    \item[] Question: Do the main claims made in the abstract and introduction accurately reflect the paper's contributions and scope?
    \item[] Answer: \answerYes{} %
    \item[] Justification: See Abstract, Section~\ref{sec:intro}, Section~\ref{sec:method} and Section~\ref{sec:experiment} %

\item {\bf Limitations}
    \item[] Question: Does the paper discuss the limitations of the work performed by the authors?
    \item[] Answer: \answerYes{} %
    \item[] Justification: See Section~\ref{sec:conclusion}%

\item {\bf Theory Assumptions and Proofs}
    \item[] Question: For each theoretical result, does the paper provide the full set of assumptions and a complete (and correct) proof?
    \item[] Answer: \answerYes{} %
    \item[] Justification: See Section~\ref{sec:method} and Appendix~\ref{app:proofs} %

    \item {\bf Experimental Result Reproducibility}
    \item[] Question: Does the paper fully disclose all the information needed to reproduce the main experimental results of the paper to the extent that it affects the main claims and/or conclusions of the paper (regardless of whether the code and data are provided or not)?
    \item[] Answer: \answerYes{} %
    \item[] Justification: See the beginning of Section~\ref{sec:experiment}%

\item {\bf Open access to data and code}
    \item[] Question: Does the paper provide open access to the data and code, with sufficient instructions to faithfully reproduce the main experimental results, as described in supplemental material?
    \item[] Answer: \answerYes{} %
    \item[] Justification: The code is not available at the time of submission, but will be published later. %

\item {\bf Experimental Setting/Details}
    \item[] Question: Does the paper specify all the training and test details (e.g., data splits, hyperparameters, how they were chosen, type of optimizer, etc.) necessary to understand the results?
    \item[] Answer: \answerYes{} %
    \item[] Justification: See the beginning of Section~\ref{sec:experiment}%

\item {\bf Experiment Statistical Significance}
    \item[] Question: Does the paper report error bars suitably and correctly defined or other appropriate information about the statistical significance of the experiments?
    \item[] Answer: \answerNo{} %
    \item[] Justification: We don't have any error bar data in our experiments. %

\item {\bf Experiments Compute Resources}
    \item[] Question: For each experiment, does the paper provide sufficient information on the computer resources (type of compute workers, memory, time of execution) needed to reproduce the experiments?
    \item[] Answer: \answerNo{} %
    \item[] Justification: Our approach utilizes established models for a range of downstream tasks, ensuring that hardware variations do not affect the outcomes. %
    
\item {\bf Code Of Ethics}
    \item[] Question: Does the research conducted in the paper conform, in every respect, with the NeurIPS Code of Ethics \url{https://neurips.cc/public/EthicsGuidelines}?
    \item[] Answer: \answerYes{} %
    \item[] Justification: the research conducted in the paper conform with the NeurIPS Code of Ethics %

\item {\bf Broader Impacts}
    \item[] Question: Does the paper discuss both potential positive societal impacts and negative societal impacts of the work performed?
    \item[] Answer: \answerYes{} %
    \item[] Justification: See Section~\ref{sec:conclusion} %
    
\item {\bf Safeguards}
    \item[] Question: Does the paper describe safeguards that have been put in place for responsible release of data or models that have a high risk for misuse (e.g., pretrained language models, image generators, or scraped datasets)?
    \item[] Answer: \answerNA{} %
    \item[] Justification: Our approach utilizes published models for a range of downstream tasks, so we do not need to add any safeguard.%

\item {\bf Licenses for existing assets}
    \item[] Question: Are the creators or original owners of assets (e.g., code, data, models), used in the paper, properly credited and are the license and terms of use explicitly mentioned and properly respected?
    \item[] Answer: \answerYes{} %
    \item[] Justification: We cited all the related works and packegs we used %

\item {\bf New Assets}
    \item[] Question: Are new assets introduced in the paper well documented and is the documentation provided alongside the assets?
    \item[] Answer: \answerYes{} %
    \item[] Justification: Our approach utilizes published models for a range of downstream tasks. Details are available in Section~\ref{sec:experiment} %

\item {\bf Crowdsourcing and Research with Human Subjects}
    \item[] Question: For crowdsourcing experiments and research with human subjects, does the paper include the full text of instructions given to participants and screenshots, if applicable, as well as details about compensation (if any)? 
    \item[] Answer: \answerNA{} %
    \item[] Justification: We don't have experiment with human subjects. %

\item {\bf Institutional Review Board (IRB) Approvals or Equivalent for Research with Human Subjects}
    \item[] Question: Does the paper describe potential risks incurred by study participants, whether such risks were disclosed to the subjects, and whether Institutional Review Board (IRB) approvals (or an equivalent approval/review based on the requirements of your country or institution) were obtained?
    \item[] Answer: \answerNA{} %
    \item[] Justification: We don't have experiment with human subjects. %

\end{enumerate}

\end{document}